\documentclass[10pt,twocolumn,letterpaper]{article}

\usepackage{cvpr}
\usepackage{times}
\usepackage{epsfig}
\usepackage{graphicx}
\usepackage{caption}
\usepackage{amsthm}
\usepackage{amsmath}
\usepackage{amssymb}
\usepackage{bm}
\usepackage{enumitem}
\usepackage{makecell}
\usepackage{wrapfig}
\usepackage{indentfirst}
\usepackage{verbatim}
\usepackage{color}
\usepackage{setspace}
\usepackage{makecell}
\usepackage{wrapfig}
\usepackage{multirow}
\usepackage{subfigure}
\usepackage{extarrows}
\usepackage{booktabs} 

\newcommand{\norm}[1]{\left\lVert#1\right\rVert}

\newtheorem{theorem}{Theorem}
\newtheorem{lemma}{Lemma}

\renewcommand{\captionlabelfont}{\scriptsize}

\usepackage[pagebackref=true,breaklinks=true,letterpaper=true,colorlinks,bookmarks=false]{hyperref}
\hypersetup{linkcolor=[rgb]{0.72,0.1,0.1}}
\hypersetup{citecolor=[rgb]{0.3,0.2,0.8}}

\cvprfinalcopy 

\def\cvprPaperID{5840} 

\ifcvprfinal\fi
\begin{document}


\title{Regularizing Neural Networks via Minimizing Hyperspherical Energy}

\author{\fontsize{10.8pt}{\baselineskip}\selectfont Rongmei Lin\textsuperscript{1}, Weiyang Liu\textsuperscript{2,*}, Zhen Liu\textsuperscript{3}, Chen Feng\textsuperscript{4}, Zhiding Yu\textsuperscript{5}, James M. Rehg\textsuperscript{2}, Li Xiong\textsuperscript{1}, Le Song\textsuperscript{2}\\
\fontsize{10pt}{\baselineskip}\selectfont \textsuperscript{1}Emory University\ \ \ \textsuperscript{2}Georgia Institute of Technology\ \ \ \textsuperscript{3}Mila, Universit\'e de Montr\'eal\ \ \ \textsuperscript{4}New York University\ \ \ \textsuperscript{5}NVIDIA\\
{\tt\small rongmei.lin@emory.edu\ \ wyliu@gatech.edu\ \ lxiong@emory.edu\ \  lsong@cc.gatech.edu} 
}

\maketitle
\thispagestyle{empty}

\begin{abstract}
\vspace{-0.75mm}
Inspired by the Thomson problem in physics where the distribution of multiple propelling electrons on a unit sphere can be modeled via minimizing some potential energy, hyperspherical energy minimization has demonstrated its potential in regularizing neural networks and improving their generalization power. In this paper, we first study the important role that hyperspherical energy plays in neural network training by analyzing its training dynamics. Then we show that naively minimizing hyperspherical energy suffers from some difficulties due to highly non-linear and non-convex optimization as the space dimensionality becomes higher, therefore limiting the potential to further improve the generalization. To address these problems, we propose the compressive minimum hyperspherical energy (CoMHE) as a more effective regularization for neural networks. Specifically, CoMHE utilizes projection mappings to reduce the dimensionality of neurons and minimizes their hyperspherical energy. According to different designs for the projection mapping, we propose several distinct yet well-performing variants and provide some theoretical guarantees to justify their effectiveness. Our experiments show that CoMHE consistently outperforms existing regularization methods, and can be easily applied to different neural networks.
\end{abstract}

\vspace{-2.5mm}
\section{Introduction}\footnote{*Weiyang Liu is the corresponding author.}
\vspace{-0.25mm}
Recent years have witnessed the tremendous success of deep neural networks in a variety of tasks. With its over-parameterization nature and hierarchical structure, deep neural networks achieve unprecedented performance on many challenging problems~\cite{he2015deep,girshick2014rich,long2015fully}, but their strong approximation ability also makes it easy to overfit the training set, which greatly affects the generalization on unseen samples. Therefore, how to restrict the huge parameter space and properly regularize the deep networks becomes increasingly important. Regularizations for neural networks can be roughly categorized into \emph{implicit} and \emph{explicit} ones. Implicit regularizations usually do not directly impose explicit constraints on neuron weights, and instead they regularize the networks in an implicit manner in order to prevent overfitting and stabilize the training. A lot of prevailing methods fall into this category, such as batch normalization~\cite{ioffe2015batch}, dropout~\cite{srivastava2014dropout}, weight normalization~\cite{salimans2016weight}, etc. Explicit regularizations~\cite{saxe2013exact,mishkin2016all,bansal2018can,rodriguez2016regularizing,huang2018orthogonal,LiuNIPS18} usually introduce some penalty terms for neuron weights, and jointly optimize them along with the other objective functions.

\begin{figure}[t]
\vspace{-0.6mm}
  \centering
  \renewcommand{\captionlabelfont}{\footnotesize}
  \setlength{\abovecaptionskip}{3pt}
  \setlength{\belowcaptionskip}{-10pt}
\includegraphics[width=3in]{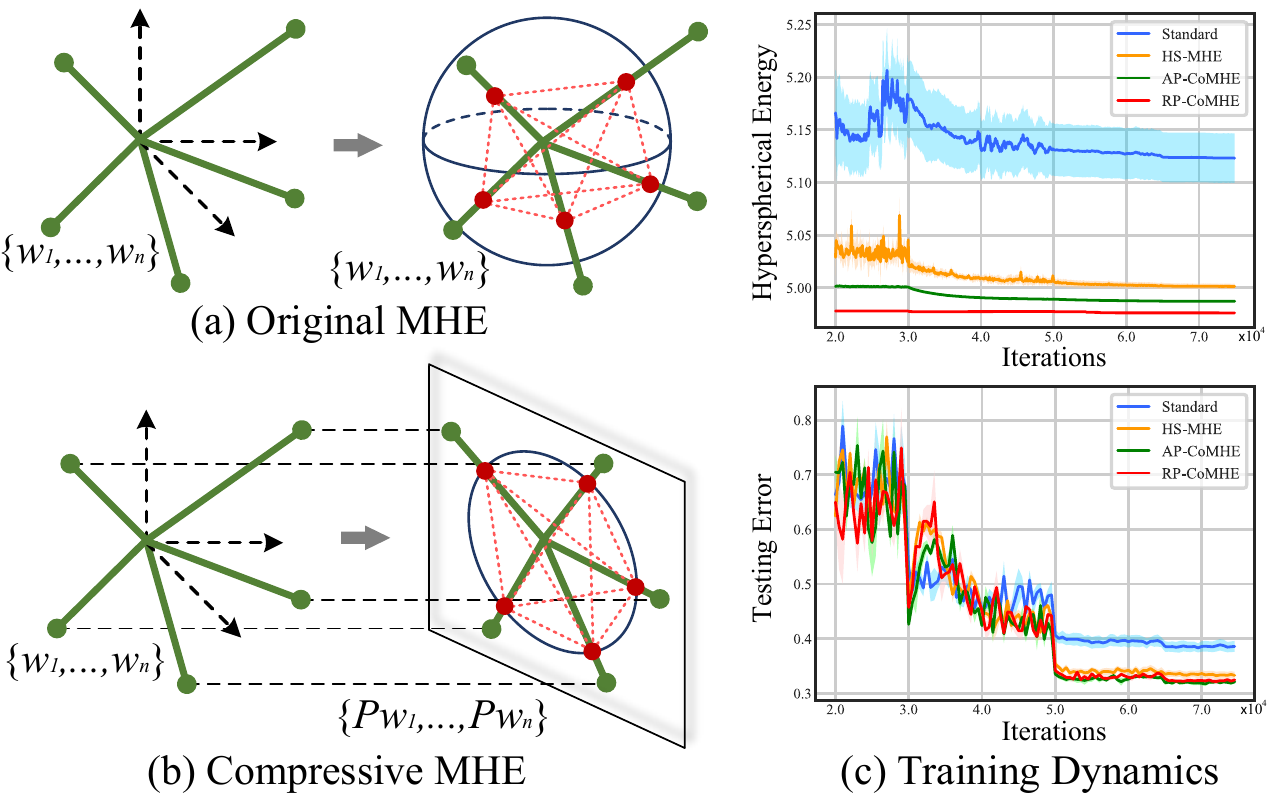}
  \caption{\footnotesize Comparison of original MHE and compressive MHE. In (c), the top figure shows the hyperspherical energy, and the bottom one shows the testing error (CIFAR-100). Experimental details are given in Appendix~\ref{exp1_detail}.}\label{rotation}
\end{figure}

\par
Among many existing explicit regularizations, minimum hyperspherical energy (MHE)~\cite{LiuNIPS18} stands out as a simple yet effective regularization that promotes the \emph{hyperspherical diversity} among neurons and significantly improves the network generalization. MHE regularizes the directions of neuron weights by minimizing a potential energy on a unit hypersphere that characterizes the hyperspherical diversity (such energy is defined as \emph{hyperspherical energy}~\cite{LiuNIPS18}). In contrast, standard weight decay only regularizes the norm of neuron weights, which essentially can be viewed as regularizing one dimension of the weights. MHE completes an important missing piece by regularizing the neuron directions (\ie, regularizing the rest dimensions of the weights).
\par
Although minimizing hyperspherical energy has already been empirically shown useful in a number of applications~\cite{LiuNIPS18}, two fundamental questions remain unanswered: \emph{(1) what is the role that hyperspherical energy plays in training a well-performing neural network?} and \emph{(2) How can the hyperspherical energy be effectively minimized?} To study the first question, we plot the training dynamics of hyperspherical energy (on CIFAR-100) in Fig.~\ref{rotation}(c) for a baseline convolutional neural network~(CNN) without any MHE variant, a CNN regularized by MHE~\cite{LiuNIPS18} and a CNN regularized by our CoMHE. More experimental details and full results (with more interesting baselines) are given in Appendix~\ref{exp1_detail}. From the empirical results in Fig.~\ref{rotation}(c), we find that both MHE and CoMHE can achieve much lower hyperspherical energy and testing error than the baseline, showing the effectiveness of minimizing hyperspherical energy. It also implies that lower hyperspherical energy typically leads to better generalization. We empirically observe that a trained neural network with lower hyperspherical energy often generalizes better (\ie, higher hyperspherical diversity leads to better generalization), and therefore we argue that hyperspherical energy is closely related to the generalization power of neural networks. In the rest of the paper, we delve into the second question that remains an open challenge: how to effectively minimize hyperspherical energy.

\par
By adopting the definition of hyperspherical energy as the regularization objective and naively minimizing it with back-propagation, MHE suffers from a few critical problems which limit it to further unleash its potential. First, the original MHE objective has a huge number of local minima and stationary points due to its highly non-convex and non-linear objective function. The problem can get even worse when the space dimension gets higher and the number of neurons becomes larger~\cite{batle2016generalized,calef2015estimating}. Second, the gradient of the original MHE objective \emph{w.r.t} the neuron weight is deterministic. Unlike the weight decay whose objective is convex, MHE has a complex and non-convex regularization term. Therefore, deterministic gradients may make the solution quickly fall into one of the bad local minima and get stuck there. Third, MHE defines an ill-posed problem in general. When the number of neurons is smaller than the dimension of the space (it is often the case in neural networks), it will be less meaningful to encourage the hyperspherical diversity since the neurons can not fully occupy the space. Last, in high-dimensional spaces, randomly initialized neurons are likely to be orthogonal to each other (see Appendix~\ref{hd_orth}). Therefore, these high-dimensional neurons can be trivially ``diverse'', leading to small gradients in original MHE that cause optimization difficulties. 
\par 
In order to address these problems and effectively minimize hyperspherical energy, we propose the compressive minimum hyperspherical energy (CoMHE) as a generic regularization for neural networks. The high-level intuition behind CoMHE is to project neurons to some suitable subspaces such that the hyperspherical energy can get minimized more effectively. Specifically, CoMHE first maps the neurons from a high-dimensional space to a low-dimensional one and then minimizes the hyperspherical energy of these neurons. Therefore, how to map these neurons to a low-dimensional space while preserving the desirable information in high-dimensional space is our major concern. Since we aim to regularize the directions of neurons, what we care most is the angular similarity between different neurons. To this end, we explore multiple novel methods to perform the projection and heavily study two main approaches: \emph{random projection} and \emph{angle-preserving projection}, which can reduce the dimensionality of neurons while still partially preserving the pairwise angles. 
\par
Random projection (RP) is a natural choice to perform the dimensionality reduction in MHE due to its simplicity and nice theoretical properties. RP can provably preserve the angular information, and most importantly, introduce certain degree of randomness to the gradients, which may help CoMHE escape from some bad local minima. The role that the randomness serves in CoMHE is actually similar to the \emph{simulated annealing}~\cite{xiang1997generalized,xiang2000efficiency} that is widely used to solve Thomson problem. Such randomness is often shown to benefit the generalization~\cite{kawaguchi2018deep,rahimi2008random}. We also provably show that using RP can well preserve the pairwise angles between neurons. Besides RP, we propose the angle-preserving projection (AP) as an effective alternative. AP is motivated by the goal that we aim to preserve the pairwise angles between neurons. Constructing an AP that can project neurons to a low-dimensional space that well preserves the angles is often difficult even with powerful non-linear functions, which is suggested by the strong conditions required for conformal mapping in complex analysis~\cite{nehari2012conformal}. Therefore, we frame the AP construction as an optimization problem which can be solved jointly with hyperspherical energy minimization. More interestingly, we consider the \emph{adversarial projection} for CoMHE, which minimizes the maximal energy attained by learning the projection. We formulate it as a min-max optimization and optimize it jointly with the neural network.
\par
However, it is inevitable to lose some information in low-dimensional spaces and the neurons may only get diverse in some particular low-dimensional spaces. To address it, we adopt multiple projections to better approximate the MHE objective in the original high-dimensional space. Specifically, we project the neurons to multiple subspaces, compute the hyperspherical energy in each space separately and then minimize the aggregation (\ie, average or max). Moreover, we reinitialize these projection matrix randomly every certain number of iterations to avoid trivial solutions. 

\par
In contrast to MHE that imposes a static regularization to the neurons, CoMHE dynamically regularizes the neurons based on the projection matrices. Such dynamic regularization is equivalent to adjusting the CoMHE objective function, making it easier to escape some bad local minima. Our contributions can be summarized as:
\par
\begin{itemize}[leftmargin=*,nosep,nolistsep]
    \item We first show that hyperspherical energy is closely related to generalization and then reveal the role it plays in training a neural network that generalizes well.
    \item To address the drawbacks of MHE, we propose CoMHE as a dynamic regularization to effectively minimize hyperspherical energy of neurons for better generalizability.
    \item We explore different ways to construct a suitable projection for CoMHE. Random projection and angle-preserving projection are proposed to reduce the dimensionality of neurons while preserving the angular information. We also consider several variants such as adversarial projection CoMHE and group CoMHE.
    \item We provide some theoretical insights for the proposed projections on the quality of preserving the angular similarity between different neurons.
    \item We show that CoMHE consistently outperforms the original MHE in different tasks. Notably, a 9-layer plain CNN regularized by CoMHE outperforms a standard 1001-layer ResNet by more than 2\% on CIFAR-100.
\end{itemize}

\vspace{-0.35mm}
\section{Related Work}
\vspace{-0.8mm}

Diversity-based regularization has been found useful in sparse coding~\cite{mairal2009online,ramirez2010classification}, ensemble learning~\cite{li2012diversity,kuncheva2003measures}, self-paced learning~\cite{jiang2014self}, metric learning~\cite{xie2018orthogonality}, latent variable models~\cite{xie2016diversity}, etc. Early studies in sparse coding~\cite{mairal2009online,ramirez2010classification} model the diversity with the empirical covariance matrix and show that encouraging such diversity can improve the dictionary's generalizability. \cite{xie2017uncorrelation} promotes the uniformity among eigenvalues of the component matrix in a latent space model. 
\cite{cogswell2016reducing,rodriguez2017regularizing,xie2017all,bansal2018can,mishkin2016all,xie2017all} characterize diversity among neurons with orthogonality, and regularize the neural network by promoting the orthogonality. Inspired by the Thomson problem in physics, MHE~\cite{LiuNIPS18} defines the hyperspherical energy to characterize the diversity on a unit hypersphere and shows significant and consistent improvement in supervised learning tasks. There are two MHE variants in \cite{LiuNIPS18}: full-space MHE and half-space MHE. Compared to full-space MHE, the half-space variant~\cite{LiuNIPS18} further eliminates the collinear redundancy by constructing virtual neurons with the opposite direction to the original ones and then minimizing their hyperspherical energy together. The importance of regularizing angular information is also discussed in \cite{liu2016large,liu2017hyper,liu2017sphereface,liu2018decoupled,deng2018arcface,wang2018cosface,wang2017normface,wang2018additive,liu2019neural,mettes2019hyperspherical}.

\vspace{-0.7mm}
\section{Compressive MHE}
\vspace{-0.4mm}
\subsection{Revisiting Standard MHE}
\vspace{-0.4mm}
MHE characterizes the diversity of $N$ neurons ($\thickmuskip=2mu \medmuskip=2mu \bm{W}_N=\{\bm{w}_1,\cdots,\bm{w}_N\in\mathbb{R}^{d+1}\}$) on a unit hypersphere using hyperspherical energy which is defined as 

\vspace{-5.3mm}
\begin{equation}\label{energy}
\scriptsize
\begin{aligned}
    \bm{E}_{s,d}(\hat{\bm{w}}_i|_{i=1}^N) &= \sum_{i=1}^{N}\sum_{j=1,j\neq i}^{N} f_s\big(\norm{\hat{\bm{w}}_i-\hat{\bm{w}}_j}\big)\\[-0.3mm]
    &=\left\{
{\begin{array}{*{20}{l}}
{\sum_{i\neq j} \norm{\hat{\bm{w}}_i-\hat{\bm{w}}_j}^{-s},\ \ s>0}\\
{\sum_{i\neq j} \log\big(\norm{\hat{\bm{w}}_i-\hat{\bm{w}}_j}^{-1}\big),\ \ s=0}
\end{array}} \right.
\end{aligned}
\vspace{-0.9mm}
\end{equation}
where $\thickmuskip=2mu \medmuskip=2mu \|\cdot\|$ denotes $\ell_2$ norm, $f_s(\cdot)$ is a decreasing real-valued function (we use $\thickmuskip=2mu \medmuskip=2mu f_s(z)=z^{-s}, s > 0$, \ie, Riesz $s$-kernels), and $\thickmuskip=2mu \medmuskip=2mu \hat{\bm{w}}_i=\frac{\bm{w}_i}{\|\bm{w}_i\|}$ is the $i$-th neuron weight projected onto the unit hypersphere $\thickmuskip=2mu \medmuskip=2mu \mathbb{S}^d=\{\bm{v}\in\mathbb{R}^{d+1}|\norm{\bm{v}}=1\}$. For convenience, we denote  $\thickmuskip=2mu \medmuskip=2mu \hat{\bm{W}}_N = \{\hat{\bm{w}}_1,\cdots,\hat{\bm{w}}_N\in\mathbb{S}^{d}\}$, and $\thickmuskip=2mu \medmuskip=2mu \bm{E}_s = \bm{E}_{s,d}(\hat{\bm{w}}_i|_{i=1}^N)$. Note that, each neuron is a convolution kernel in CNNs. MHE minimizes the hyperspherical energy of neurons using gradient descent during back-propagation, and MHE is typically applied to the neural network in a layer-wise fashion. We first write down the gradient of $\bm{E}_2$ \emph{w.r.t} $\hat{\bm{w}}_i$ and make the gradient to be zero:

\vspace{-5.1mm}
\begin{equation}\label{grad0}
\scriptsize
\nabla_{\hat{\bm{w}}_i} \bm{E}_2 = \sum_{j=1,j\neq i}^{N}  \frac{-2(\hat{\bm{w}}_i-\hat{\bm{w}}_j)}{\norm{\hat{\bm{w}}_i-\hat{\bm{w}}_j}^4}=0
\Rightarrow
\hat{\bm{w}}_i=\frac{\sum_{j=1,j\neq i}^{N}\alpha_j\hat{\bm{w}}_j}{\sum_{j=1,j\neq i}^{N}\alpha_j}
\vspace{-1.1mm}
\end{equation}
where $\thickmuskip=2mu \medmuskip=2mu \alpha_j=\|{\hat{\bm{w}}_i-\hat{\bm{w}}_j}\|^{-4}$. We use toy and informal examples to show that high dimensional space (\ie, $d$ is large) leads to much more stationary points than low-dimensional one. Assume there are $\thickmuskip=2mu \medmuskip=2mu K=K_1+K_2$ stationary points in total for $\hat{\bm{W}}_N$ to satisfy Eq.~\ref{grad0}, where $K_1$ denotes the number of stationary points in which every element in the solution is distinct and $K_2$ denotes the number of the rest stationary points. We give two examples: \emph{(i)} For $\thickmuskip=2mu \medmuskip=2mu (d+2)$-dimensional space, we can extend the solutions in $\thickmuskip=2mu \medmuskip=2mu (d+1)$-dimensional space by introducing a new dimension with zero value. The new solutions satisfy Eq.~\ref{grad0}. Because there are $\thickmuskip=2mu \medmuskip=2mu d+2$ ways to insert the zero, we have at least $\thickmuskip=2mu \medmuskip=2mu (d+2)K$ stationary points in $\thickmuskip=2mu \medmuskip=2mu (d+2)$-dimensional space. \emph{(ii)} We denote $\thickmuskip=2mu \medmuskip=2mu K'_1 = \frac{K_1}{(d+1)!}$ as the number of unordered sets that construct the stationary points. In $\thickmuskip=2mu \medmuskip=2mu (2d+2)$-dimensional space, we can construct $\thickmuskip=2mu \medmuskip=2mu \hat{\bm{w}}_j^E=\frac{1}{\sqrt{2}}\{\hat{\bm{w}}_j;\hat{\bm{w}}_j\}\in\mathbb{S}^{2d+1},\forall j$ that satisfies Eq.~\ref{grad0}. Therefore, there are at least $\thickmuskip=2mu \medmuskip=2mu \frac{(2d+2)!}{2^{d+1}} K'_1+K_2$ stationary points for $\hat{\bm{W}}_N$ in $\thickmuskip=2mu \medmuskip=2mu (2d+2)$-dimensional space, and besides this construction, there are much more stationary points. Therefore, MHE have far more stationary points in higher dimensions.

\vspace{-0.3mm}
\subsection{General Framework}
\vspace{-0.6mm}
To overcome MHE's drawbacks in high dimensional space, we propose the compressive MHE that projects the neurons to a low-dimensional space and then minimizes the hyperspherical energy of the projected neurons. In general, CoMHE minimizes the following form of energy:

\vspace{-2.1mm}
\begin{equation}\label{projected_energy}
\scriptsize
    \bm{E}^C_{s}(\hat{\bm{W}}_N):= \sum_{i=1}^{N}\sum_{j=1,j\neq i}^{N} f_s\big(\norm{g(\hat{\bm{w}}_i)-g(\hat{\bm{w}}_j)}\big)
\vspace{-0.6mm}
\end{equation}
where $\thickmuskip=2mu \medmuskip=2mu g:\mathbb{S}^{d}\rightarrow\mathbb{S}^{k}$ takes a normalized  $\thickmuskip=2mu \medmuskip=2mu (d+1)$-dimensional input and outputs a normalized $\thickmuskip=2mu \medmuskip=2mu (k+1)$-dimensional vector. $g(\cdot)$ can be either linear or nonlinear mapping. We only consider the linear case here. Using multi-layer perceptrons as $g(\cdot)$ is one of the simplest nonlinear cases. Similar to MHE, CoMHE also serves as a regularization in neural networks.
\vspace{-0.3mm}
\subsection{Random Projection for CoMHE}
\vspace{-0.6mm}
Random projection is in fact one of the most straightforward way to reduce dimensionality while partially preserving the angular information. More specifically, we use a random mapping $\thickmuskip=2mu \medmuskip=2mu g(\bm{v})=\frac{\bm{P}\bm{v}}{\|\bm{P}\bm{v}\|}$ where $\thickmuskip=2mu \medmuskip=2mu \bm{P}\in\mathbb{R}^{(k+1)\times(d+1)}$ is a Gaussian distributed random matrix (each entry follows i.i.d. normal distribution). In order to reduce the variance, we use $C$ random projection matrices to project the neurons and compute the hyperspherical energy separately:

\vspace{-3.9mm}
\begin{equation}\label{rp_energy}
\scriptsize
    \bm{E}^R_{s}(\hat{\bm{W}}_N) := \frac{1}{C}\sum_{c=1}^C\sum_{i=1}^{N}\sum_{j=1, j\neq i}^{N} f_s\big(\norm{ \frac{\bm{P}_c\hat{\bm{w}}_i}{\norm{\bm{P}_c\hat{\bm{w}}_i}} -\frac{\bm{P}_c\hat{\bm{w}}_j}{\norm{\bm{P}_c\hat{\bm{w}}_j}}}\big)
\vspace{-0.8mm}
\end{equation}
where $\thickmuskip=2mu \medmuskip=2mu\bm{P}_c,\forall c$ is a random matrix with each entry following the normal distribution $\thickmuskip=2mu \medmuskip=2mu\mathcal{N}(0,1)$. According to the properties of normal distribution~\cite{cramer2016mathematical}, every normalized row of the random matrix $\bm{P}$ is uniformly distributed on a hypersphere $\mathbb{S}^d$, which indicates that the projection matrix $\bm{P}$ is able to cover all the possible subspaces. Multiple projection matrices can also be interpreted as multi-view projection, because we are making use of information from multiple projection views. In fact, we do not necessarily need to average the energy for multiple projections, and instead we can use maximum operation (or some other meaningful aggregation operations). Then the objective becomes $\max_c\sum_{i=1}^{N}\sum_{j=1, j\neq i}^{N}f_s(\|{ \frac{\bm{P}_c\hat{\bm{w}}_i}{\norm{\bm{P}_c\hat{\bm{w}}_i}} -\frac{\bm{P}_c\hat{\bm{w}}_j}{\norm{\bm{P}_c\hat{\bm{w}}_j}}}\|)$. Considering that we aim to minimize this objective, the problem is in fact a min-max optimization. Note that, we will typically re-initialize the random projection matrices every certain number of iterations to avoid trivial solutions. Most importantly, using RP can provably preserve the angular similarity.
\par
\vspace{-0.4mm}
\subsection{Angle-preserving Projection for CoMHE}
\vspace{-0.2mm}
Recall that we aim to find a projection to project the neurons to a low-dimensional space that best preserves angular information. We transform the goal to an optimization:

\vspace{-4mm}
\begin{equation}\label{ap_p}
\scriptsize
\bm{P}^\star =\arg\min_{\bm{P}} \mathcal{L}_P:=\sum_{i\neq j} (\theta_{(\hat{\bm{w}}_i, \hat{\bm{w}}_j)} - \theta_{(\bm{P}\hat{\bm{w}}_i,\bm{P}\hat{\bm{w}}_j)})^2
\vspace{-2mm}
\end{equation}
where $\thickmuskip=2mu \medmuskip=2mu \bm{P}\in\mathbb{R}^{(k+1)\times(d+1)}$ is the projection matrix and $\theta_{(\bm{v}_1,\bm{v}_2)}$ denotes the angle between $\bm{v}_1$ and $\bm{v}_2$. For implementation convenience, we can replace the angle with the cosine value (\eg., use $\cos(\theta_{(\hat{\bm{w}}_i, \hat{\bm{w}}_j)})$ to replace $\theta_{(\hat{\bm{w}}_i, \hat{\bm{w}}_j)}$), so that we can directly use the inner product of normalized vectors to measure the angular similarity. With $\hat{\bm{P}}$ obtained in Eq.~\ref{ap_p}, we use a nested loss function:

\vspace{-5mm}
\begin{equation}\label{ap_energy}
\scriptsize
\begin{aligned}
    &\bm{E}^A_{s}(\hat{\bm{W}}_N,\bm{P}^\star) := \sum_{i=1}^{N}\sum_{j=1, j\neq i}^{N} f_s\big(\norm{ \frac{\bm{P}^\star\hat{\bm{w}}_i}{\norm{\bm{P}^\star\hat{\bm{w}}_i}} -\frac{\bm{P}^\star\hat{\bm{w}}_j}{\norm{\bm{P}^\star\hat{\bm{w}}_j}}}\big)\\
    &~~~~~~~~~\text{s.t.}~~~\bm{P}^\star =\arg\min_{\bm{P}} \sum_{i\neq j} (\theta_{(\hat{\bm{w}}_i, \hat{\bm{w}}_j)} - \theta_{(\bm{P}\hat{\bm{w}}_i,\bm{P}\hat{\bm{w}}_j)})^2
\end{aligned}
\vspace{-1mm}
\end{equation}
for which we propose two different ways to optimize the projection matrix $\bm{P}$. We can approximate $\bm{P}^\star$ using a few gradient descent updates. Specifically, we use two different ways to perform the optimization. Naively, we use a few gradient descent steps to update $\bm{P}$ in order to approximate $\bm{P}^\star$ and then update $\bm{W}_N$, which proceeds alternately. The number of iteration steps that we use to update $\bm{P}$ is a hyperparemter and needs to be determined by cross-validation. Besides the naive alternate one, we also use a different optimization of $\bm{W}_N$ by unrolling the gradient update of $\bm{P}$.
\par
\textbf{Alternating optimization.} The alternating optimization is to optimize $\bm{P}$ alternately with the network parameters $\bm{W}_N$. Specifically, in each iteration of updating the network parameters, we update $\bm{P}$ every number of inner iterations and use it as an approximation to $\bm{P}^\star$ (the error depends on the number of gradient steps we take). Essentially, we are alternately solving two separate optimization problems for $\bm{P}$ and $\bm{W}_N$ with gradient descent. 
\par
\textbf{Unrolled optimization.} Instead of naively updating $\bm{W}_N$ with approximate $\bm{P}^\star$ in the alternating optimization, the unrolled optimization further unrolls the update rule of $\bm{P}$ and embed it within the optimization of network parameters $\bm{W}_N$. If we denote the CoMHE loss with a given projection matrix $\bm{P}$ as $\bm{E}^A_{s}(\bm{W}_N,\bm{P})$ which takes $\bm{W}_N$ and $\bm{P}$ as input, then the unrolled optimization is essentially optimizing $\thickmuskip=2mu \medmuskip=2mu \bm{E}^A_{s}(\bm{W}_N,\bm{P}-\eta\cdot\frac{\partial \mathcal{L}_P}{\partial \bm{P}})$. It can also be viewed as minimizing the CoMHE loss after a single step of gradient descent \emph{w.r.t.} the projection matrix. This optimization includes the computation of second-order partial derivatives. Note that, it is also possible to unroll multiple gradient descent steps. Similar unrolling is also applied in \cite{finn2017model,liu2018darts,dai2018coupled}.

\subsection{Notable CoMHE Variants}
\vspace{-0.7mm}
We provide more interesting CoMHE variants as an extension. We will have some preliminary empirical study on these variants, but our main focus is still on RP and AP.
\par
\textbf{Adversarial Projection for CoMHE.} We consider a novel CoMHE variant that adversarially learns the projection. The intuition behind is that we want to learn a projection basis that maximizes the hyperspherical energy while the final goal is to minimize this maximal energy. With such intuition, we can construct a min-max optimization:

\vspace{-6mm}
\begin{equation}
\scriptsize
\min_{\hat{\bm{W}}_N} \max_{\bm{P}} \bm{E}^{V}_{s}(\hat{\bm{W}}_N,\bm{P})\!:=  \sum_{i=1}^{N}\sum_{j=1,j\neq i}^{N}\!f_s\big(\norm{ \frac{\bm{P}\hat{\bm{w}}_i}{\norm{\bm{P}\hat{\bm{w}}_i}} -\frac{\bm{P}\hat{\bm{w}}_j}{\norm{\bm{P}\hat{\bm{w}}_j}}}\big)
\vspace{-1mm}
\end{equation}
which can be solved by gradient descent similar to \cite{goodfellow2014generative}. From a game-theoretical perspective, $\bm{P}$ and $\hat{\bm{W}}_N$ can be viewed as two players that are competing with each other. However, due to the instability of solving the min-max problem, the performance of this projection is unstable.

\par
\textbf{Group CoMHE.} Group CoMHE is a very special case in the CoMHE framework. The basic idea is to divide the weights of each neuron into several groups and then minimize the hyperspherical energy within each group. For example in CNNs, group MHE divides the channels into groups and minimizes within each group the MHE loss. Specifically, the objective function of group CoMHE is

\vspace{-5mm}
\begin{equation}
\scriptsize
    \bm{E}^G_{s}(\hat{\bm{W}}_N) := \frac{1}{C}\sum_{c=1}^C\sum_{i=1}^{N}\sum_{j=1, j\neq i}^{N} f_s\big(\norm{ \frac{\bm{P}_c\hat{\bm{w}}_i}{\norm{\bm{P}_c\hat{\bm{w}}_i}} -\frac{\bm{P}_c\hat{\bm{w}}_j}{\norm{\bm{P}_c\hat{\bm{w}}_j}}}\big)
\vspace{-1mm}
\end{equation}
where $\bm{P}_c$ is a diagonal matrix with every diagonal entry being either $0$ or $1$, and $\thickmuskip=2mu \medmuskip=2mu \sum_c\bm{P}_c=\bm{I}$ (in fact, this is optional). There are multiple ways to divide groups for the neurons, and typically we will divide groups according to the channels, similar to \cite{wu2018group}. More interestingly, one can also divide the groups in a \emph{stochastic} fashion.

\vspace{-0.3mm}
\subsection{Shared Projection Basis in Neural Networks}
\vspace{-0.4mm}
In general, we usually need different projection bases for neurons in different layers of the neural network. However, we find it beneficial to share some projection bases across different layers. We only share the projection matrix for the neurons in different layers that have the same dimensionality. For example in a neural network, if the neurons in the first layer have the same dimensionality with the neurons in the second layer, we will share their projection matrix that reduces the dimensionality. Sharing the projection basis can effectively reduce the number of projection parameters and may also reduce the inconsistency within the hyperspherical energy minimization of projected neurons in different layers. Most importantly, it can empirically improve the network generalizability while using much fewer parameters and saving more computational overheads.
\vspace{-0.5mm}
\section{Theoretical Insights}
\vspace{-0.3mm}

\subsection{Angle Preservation}
\vspace{-0.3mm}
We start with highly relevant properties of random projection and then delve into the angular preservation.
\vspace{-0.5mm}
\begin{lemma}[Mean Preservation of Random Projection]\label{lem1}
For any $\bm{w}_1,\bm{w}_2\in\mathbb{R}^d$ and any random Gaussian distributed matrix $\bm{P}\in\mathbb{R}^{k\times d}$ where $\bm{P}_{ij}=\frac{1}{\sqrt{n}}r_{ij}$, if $r_{ij},\forall i,j$ are i.i.d. random variables from $\mathcal{N}(0,1)$, we have $\mathbb{E}(\langle \bm{P}\bm{w}_1,\bm{P}\bm{w}_2 \rangle)=\langle \bm{w}_1,\bm{w}_2\rangle$.
\end{lemma}
\vspace{-0.5mm}
\par
This lemma indicates that the mean of randomly projected inner product is well preserved, partially justifying why using random projection actually makes senses.
\par
Johnson-Lindenstrauss lemma (JLL)~\cite{dasgupta2003elementary,kaban2015improved} (in Appendix~\ref{app_jll}) establishes a guarantee for the Euclidean distance between randomly projected vectors. However, JLL does not provide the angle preservation guarantees. It is nontrivial to provide a guarantee for angular similarity from JLL.
\par
\vspace{-0.6mm}
\begin{theorem}[Angle Preservation I]\label{thm1}
Given $\bm{w}_1,\bm{w}_2\in\mathbb{R}^d$, $\bm{P}\in\mathbb{R}^{k\times d}$ is a random projection matrix that has i.i.d. $0$-mean $\sigma$-subgaussian entries, and $\bm{P}\bm{w}_1,\bm{P}\bm{w}_2\in\mathbb{R}^k$ are the randomly projected vectors of $\bm{w}_1,\bm{w}_2$ under $\bm{P}$. Then $\forall \epsilon\in(0,1)$, we have that
\vspace{-5mm}

\begin{equation}
\scriptsize
\begin{aligned}
    \frac{\cos(\theta_{(\bm{w}_1,\bm{w}_2)})-\epsilon}{1+\epsilon}<\cos(\theta_{(\bm{P}\bm{w}_1,\bm{P}\bm{w}_2)})
    <\frac{\cos(\theta_{(\bm{w}_1,\bm{w}_2)})+\epsilon}{1-\epsilon}
\end{aligned}
\vspace{-1.4mm}
\end{equation}
which holds with probability $\big(1-2\exp({-\frac{k\epsilon^2}{8}})\big)^2$. 
\end{theorem}
\vspace{-2.2mm}
\par
\begin{theorem}[Angle Preservation II]
Given $\bm{w}_1,\bm{w}_2\in\mathbb{R}^d$, $\bm{P}\in\mathbb{R}^{k\times d}$ is a Gaussian random projection matrix where $\thickmuskip=2mu \medmuskip=2mu \bm{P}_{ij}=\frac{1}{\sqrt{n}}r_{ij}$ ($r_{ij},\forall i,j$ are i.i.d. random variables from $\mathcal{N}(0,1)$), and $\bm{P}\bm{w}_1,\bm{P}\bm{w}_2\in\mathbb{R}^k$ are the randomly projected vectors of $\bm{w}_1,\bm{w}_2$ under $\bm{P}$. Then $\forall \epsilon\in(0,1)$ and $\bm{w}_1^\top\bm{w}_2>0$, we have that
\vspace{-5.4mm}

\begin{equation}
\scriptsize
\begin{aligned}
    &\frac{1+\epsilon}{1-\epsilon}\cos(\theta_{(\bm{w}_1,\bm{w}_2)})-
    \frac{2\epsilon}{1-\epsilon}<\cos(\theta_{(\bm{P}\bm{w}_1,\bm{P}\bm{w}_2)})\\
    &\ \ \ \ \ \ \ \ \ \ \ \ \ \ \ \ \ <\frac{1-\epsilon}{1+\epsilon}\cos(\theta_{(\bm{w}_1,\bm{w}_2)})+\frac{1+2\epsilon}{1+\epsilon}-\frac{\sqrt{(1-\epsilon^2)}}{1+\epsilon}
\end{aligned}
\vspace{-1.65mm}
\end{equation}
which holds with probability $1-6\exp(-\frac{k}{2}(\frac{\epsilon^2}{2}-\frac{\epsilon^3}{3}))$. 
\label{thm2}\end{theorem}
\vspace{-0.3mm}
\par
Theorem~\ref{thm1} is one of our main theoretical results and reveals that the angle between randomly projected vectors is well preserved. Note that, the parameter $\sigma$ of the subgaussian distribution is not related to our bound for the angle, so any Gaussian distributed random matrix has the property of angle preservation. The projection dimension $k$ is related to the probability that the angle preservation bound holds. Theorem~\ref{thm2} is a direct result from \cite{shi2012margin}. It again shows that the angle between randomly projected vectors is provably preserved. Both Theorem~\ref{thm1} and Theorem~\ref{thm2} give upper and lower bounds for the angle between randomly projected vectors. If $\thickmuskip=2mu \medmuskip=2mu \theta_{(\bm{w}_1,\bm{w}_2)}>\arccos(\frac{\epsilon+3\epsilon^2}{3\epsilon+\epsilon^2})$, then the lower bound in Theorem~\ref{thm1} is tighter than the lower bound in Theorem~\ref{thm2}. If $\thickmuskip=2mu \medmuskip=2mu \theta_{(\bm{w}_1,\bm{w}_2)}>\arccos(\frac{1-3\epsilon^2-(1-\epsilon)\sqrt{1-\epsilon^2}}{3\epsilon-\epsilon^2})$, the upper bound in Theorem~\ref{thm1} is tighter than the upper bound in Theorem~\ref{thm2}. To conclude, Theorem~\ref{thm1} gives tighter bounds when the angle of original vectors is large. Since AP is randomly initialized every certain number of iterations and minimizes the angular difference before and after the projection, AP usually performs better than RP in preserving angles. Without the angle-preserving optimization, AP reduces to RP.

\vspace{-0.15mm}
\subsection{Statistical Insights}
\vspace{-0.25mm}
We can also draw some theoretical intuitions from spherical uniform testing~\cite{cuesta2009projection} in statistics. Spherical uniform testing is a nonparametric statistical hypothesis test that checks whether a set of observed data is generated from a uniform distribution on a hypersphere or not. Random projection is in fact an important tool~\cite{cuesta2009projection} in statistics to test the uniformity on hyperspheres, while our goal is to promote the same type of hyperspherical uniformity (\ie, diversity). Specifically, we have $N$ random samples $\thickmuskip=2mu \medmuskip=2mu \bm{w}_1,\cdots,\bm{w}_N$ of $\mathbb{S}^d$-valued random variables, and the random projection $\bm{p}$ which is another random variable independent of $\bm{w}_i,\forall i$ and uniformly distributed on $\mathbb{S}^d$. The projected points of $\bm{w}_i,\forall i$ is $\thickmuskip=2mu \medmuskip=2mu y_i=\bm{p}^\top \bm{w}_i,\forall i$. The distribution of $y_i,\forall i$ uniquely determines the distribution of $\bm{w}_1$, as is specified by Theorem~\ref{thm3}.
\par
\vspace{-1.5mm}
\begin{theorem}[Unique Distribution Determination of Random Projection]
Let $\bm{w}$ be a $\mathbb{S}^d$-valued random variable and $\bm{p}$ be a random variable that is uniformly distributed on $\mathbb{S}^d$ and independent of $\bm{w}$. With probability one, the distribution of $\bm{w}$ is uniquely determined by the distribution of the projection of $\bm{w}$ on $\bm{p}$. More specifically, if $\bm{w}_1$ and $\bm{w}_2$ are $\mathbb{S}^d$-valued random variables, independent of $\bm{p}$ and we have a positive probability for the event that $\bm{p}$ takes a value $\bm{p}_0$ such that the two distributions satisfy $\bm{p}_0^\top\bm{w}_1\sim\bm{p}_0^\top\bm{w}_2$, then $\bm{w}_1$ and $\bm{w}_2$ are identically distributed.
\label{thm3}\end{theorem}
\vspace{-1.5mm}
\par
Theorem~\ref{thm3} shows that the distributional information is well preserved after random projection, providing the CoMHE framework a statistical intuition and foundation. We emphasize that the randomness here is in fact very crucial. For a fixed projection $\bm{p}_0$, Theorem~\ref{thm3} does not hold in general. As a result, random projection for CoMHE is well motivated from the statistical perspective.
\vspace{-0.15mm}
\subsection{Insights from Random Matrix Theory}
\vspace{-0.25mm}
Random projection may also impose some implicit regularization to learning the neuron weights. \cite{durrant2015random} proves that random projection serves as a regularizer for the Fisher linear discrimination classifier. From metric learning perspective, the inner product between neurons $\bm{w}_1^\top\bm{w}_2$ will become $\bm{w}_1^\top\bm{P}^\top\bm{P}\bm{w}_2$ where $\bm{P}^\top\bm{P}$ defines a specific form of (low-rank) similarity~\cite{xing2003distance,liu2019neural}. \cite{baraniuk2008simple} proves that random projection satisfying the JLL \emph{w.h.p}
also satisfies the restricted isometry property~(RIP) \emph{w.h.p} under sparsity assumptions. In this case, the neuron weights can be well recovered~\cite{plan2013one,candes2006robust}. These results imply that randomly projected neurons in CoMHE may implicitly regularize the network.

\vspace{-1.5mm}
\section{Discussions and Extensions}
\vspace{-0.8mm}
\textbf{Bilateral projection for CoMHE.} If we view the neurons in one layer as a matrix $\thickmuskip=2mu \medmuskip=2mu \bm{W}=\{\bm{w}_1,\cdots,\bm{w}_n\}\in\mathbb{R}^{m\times n}$ where $m$ is the dimension of neurons and $n$ is the number of neurons, then the projection considered throughout the paper is to left-multiply a projection matrix $\thickmuskip=2mu \medmuskip=2mu \bm{P}_1\in\mathbb{R}^{r\times m}$ to $\bm{W}$. In fact, we can further reduce the number of neurons by right-multiplying an additional projection matrix $\thickmuskip=2mu \medmuskip=2mu \bm{P}_2\in\mathbb{R}^{n\times r}$ to $\bm{W}$. Specifically, we denote that $\thickmuskip=2mu \medmuskip=2mu \bm{Y}_1=\bm{P}_1\bm{W}$ and $\thickmuskip=2mu \medmuskip=2mu \bm{Y}_2=\bm{W}\bm{P}_2$. Then we can apply the MHE regularization separately to column vectors of $\bm{Y}_1$ and $\bm{Y}_2$. The final neurons are still $\bm{W}$. More interestingly, we can also approximate $\bm{W}$ with a low-rank factorization~\cite{zhou2011godec}: $\thickmuskip=2mu \medmuskip=2mu \tilde{\bm{W}}= \bm{Y}_2(\bm{P}_1\bm{Y}_2)^{-1}\bm{Y}_1\in\mathbb{R}^{m\times n}$. It inspires us to directly use two set of parameters $\bm{Y}_1$ and $\bm{Y}_2$ to represent the equivalent neurons $\tilde{\bm{W}}$ and apply the MHE regularization separately to their column vectors. Different from the former case, we use $\tilde{\bm{W}}$ as the final neurons. More details are in Appendix~\ref{biproj}.

\textbf{Constructing random projection matrices.} In random projection, we typically construct random matrices with each element drawn \emph{i.i.d.} from a normal distribution. However, there are many more choices for constructing a random matrices that can provably preserve distance information. For example, we have subsampled randomized Hadamard transform~\cite{ailon2006approximate} and count sketch-based projections~\cite{charikar2004finding}.

\textbf{Comparison to existing works.} One of the widely used regularizations is the orthonormal regularization~\cite{liu2017hyper,brock2016neural} that minimizes $\thickmuskip=2mu \medmuskip=2mu \|\bm{W}^\top\bm{W}-\bm{I}\|_F$ where $W$ denotes the weights of a group of neurons with each column being one neuron and $\bm{I}$ is an identity matrix. \cite{bansal2018can,rodriguez2017regularizing} are also built upon orthogonality. In contrast, both MHE and CoMHE do not encourage orthogonality among neurons and instead promote hyperspherical uniformity and diversity. 
\par

\par
\textbf{Randomness improves generalization.} Both RP and AP introduce randomness to CoMHE, and the empirical results show that such randomness can greatly benefit the network generalization. It is well-known that stochastic gradient is one of the key ingredients that help neural networks generalize well to unseen samples. Interestingly, randomness in CoMHE also leads to a stochastic gradient. \cite{kawaguchi2018deep} also theoretically shows that randomness helps generalization, partially justifying the effectiveness of CoMHE.

\vspace{-1.2mm}
\section{Experiments and Results}\label{expandresult}
\vspace{-0.6mm}
\subsection{Image Recognition}
\vspace{-1mm}
We perform image recognition to show the improvement of regularizing CNNs with CoMHE. Our goal is to show the superiority of CoMHE rather than achieving state-of-the-art accuracies on particular tasks. For all the experiments on CIFAR-10 and CIFAR-100 in the paper, we use the same data augmentation as \cite{he2015deep,liu2018decoupled}. For ImageNet-2012, we use the same data augmentation in \cite{liu2017hyper}. We train all the networks using SGD with momentum 0.9. All the networks use BN~\cite{ioffe2015batch} and ReLU if not otherwise specified. By default, all CoMHE variants are built upon half-space MHE. Experimental details are given in each subsection and Appendix~\ref{arch}. More experiments are given in Appendix~\ref{exp_time},\ref{extra_explor},\ref{graph}.

\vspace{-3mm}
\subsubsection{Ablation Study and Exploratory Experiments}\label{cifar_explore}
\vspace{-1mm}

\setlength{\columnsep}{9pt}
\begin{wraptable}{r}[0cm]{0pt}
	\centering
	\scriptsize
	\newcommand{\tabincell}[2]{\begin{tabular}{@{}#1@{}}#2\end{tabular}}
	\setlength{\tabcolsep}{2pt}
\begin{tabular}{c||c}
\specialrule{0em}{-11pt}{0pt}
  \hline
  \multirow{1}{*}{Method} & Error (\%)  \\
 \hline\hline
 Baseline & 28.03 \\
 Orthogonal & 27.01 \\
 SRIP~\cite{bansal2018can} & 25.80\\
 MHE~\cite{LiuNIPS18}& 26.75 \\
 HS-MHE~\cite{LiuNIPS18}   & 25.96 \\\hline\hline
 G-CoMHE & 25.08 \\
 Adv-CoMHE & 25.09\\
 RP-CoMHE & 24.39 \\
 RP-CoMHE (max) & 24.77 \\
 AP-CoMHE (alter.) & 24.95\\
 AP-CoMHE (unroll) & \textbf{24.33}\\
 \hline
 \specialrule{0em}{0pt}{-7pt}
\end{tabular}
\renewcommand{\captionlabelfont}{\footnotesize}
\caption{\scriptsize CoMHE variants on C-100.}
\label{CoMHE_variant}
\vspace{-3mm}
\end{wraptable}

\textbf{Variants of CoMHE}. We compare different variants of CoMHE with the same plain CNN-9 (Appendix~\ref{arch}). Specifically, we evaluate the baseline CNN without any regularization, half-space MHE (HS-MHE) which is the best MHE variant from \cite{LiuNIPS18}, random projection CoMHE (RP-CoMHE), RP-CoMHE (max) that uses max instead of average for loss aggregation, angle-preserving projection CoMHE (AP-CoMHE), adversarial projection CoMHE (Adv-CoMHE) and group CoMHE (G-CoMHE) on CIFAR-100. For RP, we set the projection dimension to 30 (\ie, $\thickmuskip=2mu \medmuskip=2mu k=29$) and the number of projection to 5 (\ie, $\thickmuskip=2mu \medmuskip=2mu C=5$). For AP, the number of projection is 1 and the projection dimension is set to 30. For AP, we evaluate both alternating optimization and unrolled optimization. In alternating optimization, we update the projection matrix every 10 steps of network update. In unrolled optimization, we only unroll one-step gradient in the optimization. For G-CoMHE, we construct a group with every 8 consecutive channels. All these design choices are obtained using cross-validation. We will also study how these hyperparameters affect the performance in the following experiments. The results in Table~\ref{CoMHE_variant} show that all of our proposed CoMHE variants can outperform the original half-space MHE by a large margin. The unrolled optimization in AP-CoMHE shows the significant advantage over alternating one and achieves the best accuracy. Both Adv-CoMHE and G-CoMHE achieve decent performance gain over HS-MHE, but not as good as RP-CoMHE and AP-CoMHE. Therefore, we will mostly focus on RP-CoMHE and AP-CoMHE in the remaining experiments.
\par

\begin{table}[h]
\scriptsize
  \setlength{\abovecaptionskip}{2pt}
  \setlength{\belowcaptionskip}{-7pt}
\centering
\begin{tabular}{c || c c c c c} 
\specialrule{0em}{-6pt}{0pt}
  \hline
  Projection Dimension & 10 & 20 & 30 & 40 & 80 \\
 \hline\hline
RP-CoMHE& 25.48 & 25.32 & 24.60 & \textbf{24.75} & 25.46\\
AP-CoMHE (alter.)& \textbf{25.21} & 24.60 & 24.95 & 24.97 & \textbf{24.99}\\
AP-CoMHE (unroll.)& 25.32 & \textbf{24.59} & \textbf{24.33} & 24.93 & 25.12\\
 \hline
\end{tabular}
\caption{\scriptsize Error (\%) on CIFAR-100 under different dimension of projection.}
\label{proj_dim}
\end{table}

\textbf{Dimension of projection}. We evaluate how the dimension of projection (\ie, $k$) affects the performance. We use the plain CNN-9 as the backbone network and test on CIFAR-100. We fix the number of projections in RP-CoMHE to 20. Because AP-CoMHE does not need to use multiple projections to reduce variance, we only use one projection in AP-CoMHE. Results are given in Table~\ref{proj_dim}. In general, RP-CoMHE and AP-CoMHE with different projection dimensions can consistently and significantly outperform the half-space MHE, validating the effectiveness and superiority of the proposed CoMHE framework. Specifically, we find that both RP-CoMHE and AP-CoMHE usually achieve the best accuracy when the projection dimension is 20 or 30. Since the unrolled optimization in AP-CoMHE is consistently better than the alternating optimization, we  stick to the unrolled optimization for AP-CoMHE in the remaining experiments if not otherwise specified.
\par

\setlength{\columnsep}{9pt}
\begin{wraptable}{r}[0cm]{0pt}
	\centering
	\scriptsize
	\newcommand{\tabincell}[2]{\begin{tabular}{@{}#1@{}}#2\end{tabular}}
	\setlength{\tabcolsep}{2pt}
\begin{tabular}{c || c c} 
\specialrule{0em}{-9pt}{0pt}
  \hline
  \# Proj.& RP-CoMHE& AP-CoMHE\\
 \hline\hline
1 & 25.11 & \textbf{24.33}\\
5 & \textbf{24.39} & 24.34\\
10 & 25.11 & 24.36\\ 
20 & 24.60 &  24.38\\
30 & 24.82 & 24.52\\
80 & 24.92 & 24.56\\
 \hline
 \specialrule{0em}{0pt}{-8pt}
\end{tabular}
\caption{\scriptsize Error (\%) on CIFAR-100 under different numbers of projections.}
\label{no_proj}
\vspace{-3mm}
\end{wraptable}

\textbf{Number of projections}. We evaluate RP-CoMHE under different numbers of projections. We use the plain CNN-9 as the baseline and test on CIFAR-100. Results in Table~\ref{no_proj} show that the performance of RP-CoMHE is generally not very sensitive to the number of projections. Surprisingly, we find that it is not necessarily better to use more projections for variance reduction. Our experiment show that using 5 projections can achieve the best accuracy. It may be because large variance can help the solution escape bad local minima in the optimization. Note that, we generally do not use multiple projections in AP-CoMHE, because AP-CoMHE optimizes the projection and variance reduction is unnecessary. Our results do not show performance gain by using multiple projections in AP-CoMHE.

\begin{table}[h]
\scriptsize
  \setlength{\abovecaptionskip}{1.5pt}
  \setlength{\belowcaptionskip}{-18pt}
  \setlength{\tabcolsep}{5pt}
\centering
\begin{tabular}{c || c c c c c c} 
\specialrule{0em}{5pt}{0pt}
  \hline
  Width & $t=1$ & $t=2$ & $t=4$ & $t=8$ & $t=16$ & $t=20$\\
 \hline\hline
 Baseline & 47.72 & 38.64 & 28.13 & 24.95 & 24.44 &23.77\\
 MHE~\cite{LiuNIPS18} & 36.84 & 30.05 & 26.75 & 24.05 & 23.14 & 22.36\\
HS-MHE~\cite{LiuNIPS18} & 35.16 & 29.33 & 25.96 & 23.38 & 21.83 & 21.22\\\hline\hline
 RP-CoMHE & \textbf{34.73} & \textbf{28.92} & 24.39 & \textbf{22.44} & 20.81 & 20.62\\
 AP-CoMHE & 34.89 & 29.01 & \textbf{24.33} & 22.6 & \textbf{20.72}& \textbf{20.50}\\
 \hline
\end{tabular}
\caption{\scriptsize Error (\%) on CIFAR-100 with different network width.}
\label{width}
\end{table}

\par
\textbf{Network width.} We evaluate RP-CoMHE and AP-CoMHE with different network width on CIFAR-100. We use the plain CNN-9 as our backbone network architecture, and set its filter number in Conv1.x, Conv2.x and Conv3.x (see Appendix~\ref{arch}) to $16\times t$, $32\times t$ and $64\times t$, respectively. Specifically, we test the cases where $t=1,2,4,8,16$. Taking $t=2$ as an example, then the filter numbers in Conv1.x, Conv2.x and Conv3.x are 32, 64 and 128, respectively. For RP, we set the projection dimension to 30 and the number of projection to 5. For AP, the number of projection is 1 and the projection dimension is set to 30. The results are shown in Table~\ref{width}. Note that, we use the unrolled optimization in AP-CoMHE. From Table~\ref{width}, one can observe that the performance gains of both RP-CoMHE and AP-CoMHE are very consistent and significant. With wider network, CoMHE also achieves better accuracy. Compared to the strong results of half-space MHE, CoMHE still obtains more than 1\% accuracy boost under different network width.

\vspace{0.5mm}

\setlength{\columnsep}{10pt}
\begin{wraptable}{r}[0cm]{0pt}
	\centering
	\scriptsize
	\newcommand{\tabincell}[2]{\begin{tabular}{@{}#1@{}}#2\end{tabular}}
	\setlength{\tabcolsep}{2.6pt}
\begin{tabular}{c || c c c } 
\specialrule{0em}{-11pt}{0pt}
  \hline
  Depth & CNN-6 & CNN-9 & CNN-15 \\
 \hline\hline
 Baseline & 32.08 & 28.13 & N/C\\
 MHE~\cite{LiuNIPS18} & 28.16 & 26.75 & 26.90 \\
 HS-MHE~\cite{LiuNIPS18} & 27.56 & 25.96 & 25.84\\\hline\hline
 RP-CoMHE & 26.73 & 24.39 & \textbf{24.21}\\
 AP-CoMHE & \textbf{26.55} & \textbf{24.33} & 24.55 \\
 \hline
  \specialrule{0em}{0pt}{-9pt}
\end{tabular}
\caption{\scriptsize Error on CIFAR-100 with different network depth. N/C denotes Not Converged.}
\label{depth}
\vspace{-2.5mm}
\end{wraptable}

\par
\textbf{Network depth.} We evaluate RP-CoMHE and AP-CoMHE with different network depth on CIFAR-100. We use three plain CNNs with 6, 9 and 15 convolution layers, respectively. For all the networks, we set the filter number in Conv1.x, Conv2.x and Conv3.x to 64, 128 and 256, respectively. Detailed network architectures are given in Appendix~\ref{arch}. For RP, we set the projection dimension to 30 and the number of projection to 5. For AP, the number of projection is 1 and the projection dimension is set to 30. Table~\ref{depth} shows that both RP-CoMHE and AP-CoMHE can outperform half-space MHE by a considerable margin while regularizing a plain CNN with different depth.

\setlength{\columnsep}{9pt}
\begin{wrapfigure}{r}{0.251\textwidth}
  \begin{center}
  \advance\leftskip+1mm
  \renewcommand{\captionlabelfont}{\scriptsize}
    \vspace{-0.29in}  
    \includegraphics[width=0.253\textwidth]{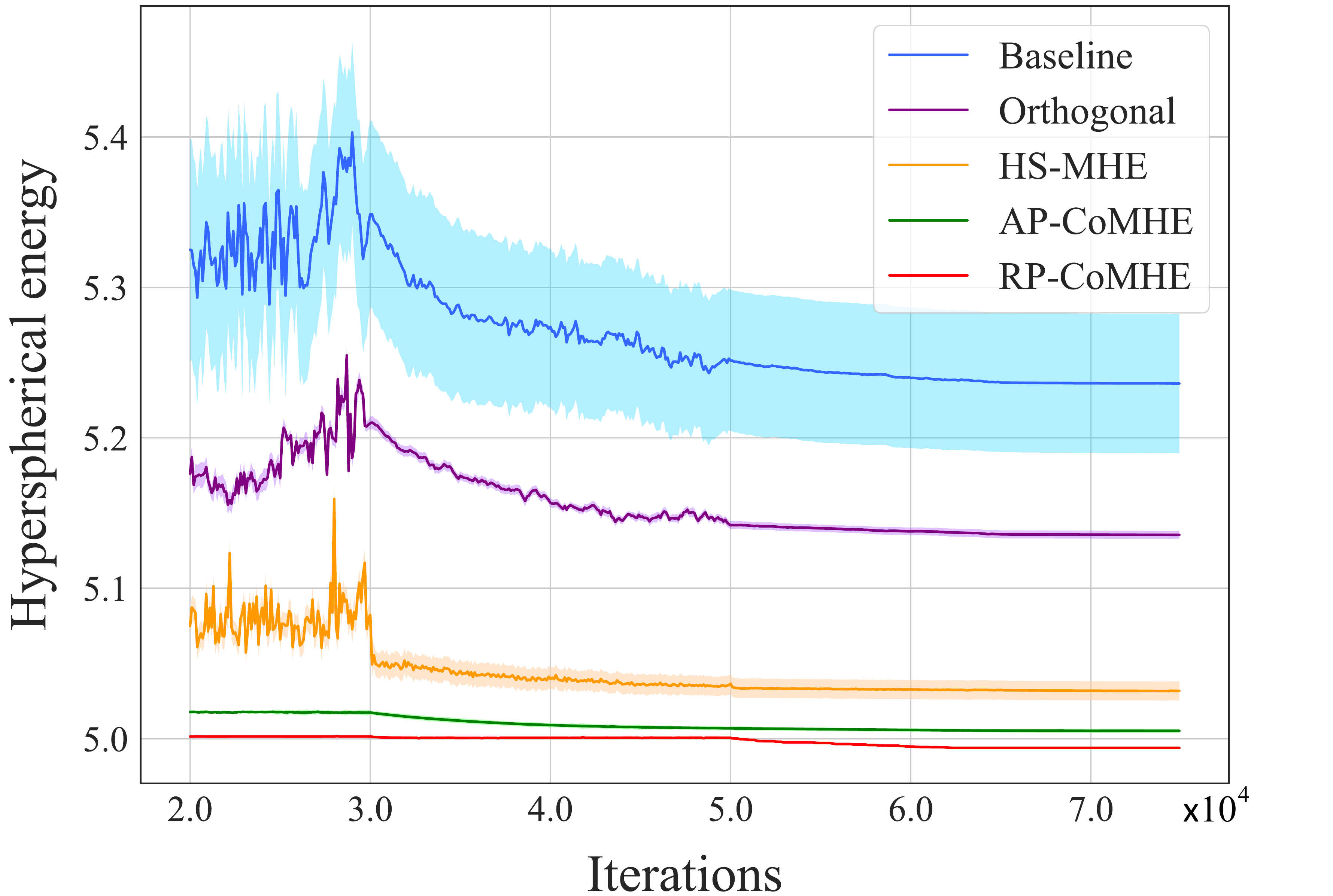}
    \vspace{-0.28in} 
    \caption{\scriptsize Hyperspherical energy during training. All networks are initialized with the same random weights, so the hyperspherical energy is the same before the training starts.}\label{training}
    \vspace{-0.2in} 
  \end{center}
\end{wrapfigure}
\par
\textbf{Effectiveness of optimization.} To verify that our CoMHE can better minimize the hyperspherical energy, we compute the hyperspherical energy $\bm{E}_2$ (Eq.~\ref{energy}) for baseline CNN and CNN regularized by orthogonal regularization, HS-MHE, RP-CoMHE and AP-CoMHE during training. Note that, we compute the original hyperspherical energy rather than the energy after projection. All the networks use exactly the same initialization (the initial hyperspherical energy is the same). The results are averaged over five independent runs. We show the hyperspherical energy after the 20000-th iteration, because at the beginning of the training, hyperspherical energy fluctuates dramatically and is unstable. Interested readers can refer to Appendix~\ref{effectiveness} for the complete energy dynamics. From Fig.~\ref{training}, one can observe that both RP-CoMHE and AP-CoMHE can better minimize the hyperspherical energy. RP-CoMHE can achieve the lowest energy with smallest standard deviation. From the absolute scale, the optimization gain is also very significant. In the high-dimensional space, the variance of hyperspherical energy is usually small (already close to the smallest energy value) and is already difficult to minimize.
\par
\vspace{0.3mm}

\setlength{\columnsep}{9pt}
\begin{wraptable}{r}[0cm]{0pt}
	\centering
	\scriptsize
	\newcommand{\tabincell}[2]{\begin{tabular}{@{}#1@{}}#2\end{tabular}}
	\setlength{\tabcolsep}{3pt}
\begin{tabular}{c || c c } 
\specialrule{0em}{-10pt}{0pt}
  \hline
  \multirow{1}{*}{Method} & C-10 & C-100  \\\hline\hline
  ResNet-110~\cite{he2015deep} & 6.61 & 25.16\\
  ResNet-1001~\cite{he2016identity} & 4.92 & \textbf{22.71}\\
 \hline\hline
 Baseline & 5.19 & 22.87\\
 Orthogonal~\cite{rodriguez2017regularizing}& 5.02 & 22.36 \\
 SRIP~\cite{bansal2018can} & 4.75 & 22.08\\
 MHE~\cite{LiuNIPS18}      & 4.72 & 22.19  \\
 HS-MHE~\cite{LiuNIPS18}   & 4.66 & 22.04 \\
 RP-CoMHE   & 4.59 & 21.82 \\
 AP-CoMHE   & \textbf{4.57} & \textbf{21.63} \\
 \hline
 \specialrule{0em}{0pt}{-9pt}
\end{tabular}
\caption{\scriptsize Error (\%) using ResNets.}
\label{cifar}
\vspace{-2.5mm}
\end{wraptable}

\textbf{ResNet with CoMHE}. All the above experiments are performed using VGG-like plain CNNs, so we use the more powerful ResNet~\cite{he2015deep} to show that CoMHE is architecture-agnostic. We use the same experimental setting in \cite{he2016identity} for fair comparison. We use a standard ResNet-32 as our baseline and the network architecture is specified in Appendix~\ref{arch}. From the results in Table~\ref{cifar}, one can observe that both RP-CoMHE and AP-CoMHE can consistently outperform half-space MHE, showing that CoMHE can boost the performance across different network architectures. More interestingly, the ResNet-32 regularized by CoMHE achieves impressive accuracy and is able to outperform the 1001-layer ResNet by a large margin. Additionally, we note that from Table~\ref{width}, we can regularize a plain VGG-like 9-layer CNN with CoMHE and achieve 20.81\% error rate, which is nearly 2\% improvement over the 1001-layer ResNet.

\vspace{-3.5mm}
\subsubsection{Large-scale Recognition on ImageNet-2012}
\vspace{-1.5mm}

\setlength{\columnsep}{9pt}
\begin{wraptable}{r}[0cm]{0pt}
	\centering
	\scriptsize
	\newcommand{\tabincell}[2]{\begin{tabular}{@{}#1@{}}#2\end{tabular}}
	\setlength{\tabcolsep}{2pt}
\begin{tabular}{c || c c c} 
\specialrule{0em}{-10pt}{0pt}
  \hline
  \multirow{1}{*}{Method} & Res-18 & Res-34 & Res-50\\
 \hline\hline
 baseline & 32.95 & 30.04 & 25.30\\
 Orthogonal~\cite{rodriguez2017regularizing}& 32.65 & 29.74 & 25.19\\
 Orthnormal~\cite{liu2017hyper} & 32.61 & 29.75 & 25.21\\
  SRIP~\cite{bansal2018can} & 32.53 & 29.55 & 24.91\\
 MHE~\cite{LiuNIPS18}      & 32.50 & 29.60  & 25.02\\
 HS-MHE~\cite{LiuNIPS18}   & 32.45 & 29.50 & 24.98\\\hline\hline
 RP-CoMHE   & 31.90 & 29.38 & \textbf{24.51}\\
 AP-CoMHE   & \textbf{31.80} & \textbf{29.32} & 24.53\\
 \hline
 \specialrule{0em}{-9pt}{0pt}
\end{tabular}
\captionof{table}{\scriptsize Top-1 center crop error on ImageNet.}
\label{imagenet}
\vspace{-2.5mm}
\end{wraptable}

We evaluate CoMHE for image recognition on ImageNet-2012~\cite{russakovsky2014imagenet}. We perform the experiment using ResNet-18, ResNet-34 and ResNet-50, and then report the top-1 validation error (center crop) in Table~\ref{imagenet}. Our results show consistent and significant performance gain of CoMHE in all ResNet variants. Compared to the baselines, CoMHE can reduce the top-1 error for more than 1\%. Since the computational overhead of CoMHE is almost neglectable, the performance gain is obtained without many efforts. Most importantly, as a plug-in regularization, CoMHE is shown to be architecture-agnostic and produces considerable accuracy gain in most circumstances.

\setlength{\columnsep}{10pt}
\begin{wrapfigure}{r}{0.225\textwidth}
  \begin{center}
  \advance\leftskip+1mm
  \renewcommand{\captionlabelfont}{\scriptsize}
    \vspace{-0.155in} 
    \includegraphics[width=1.47in]{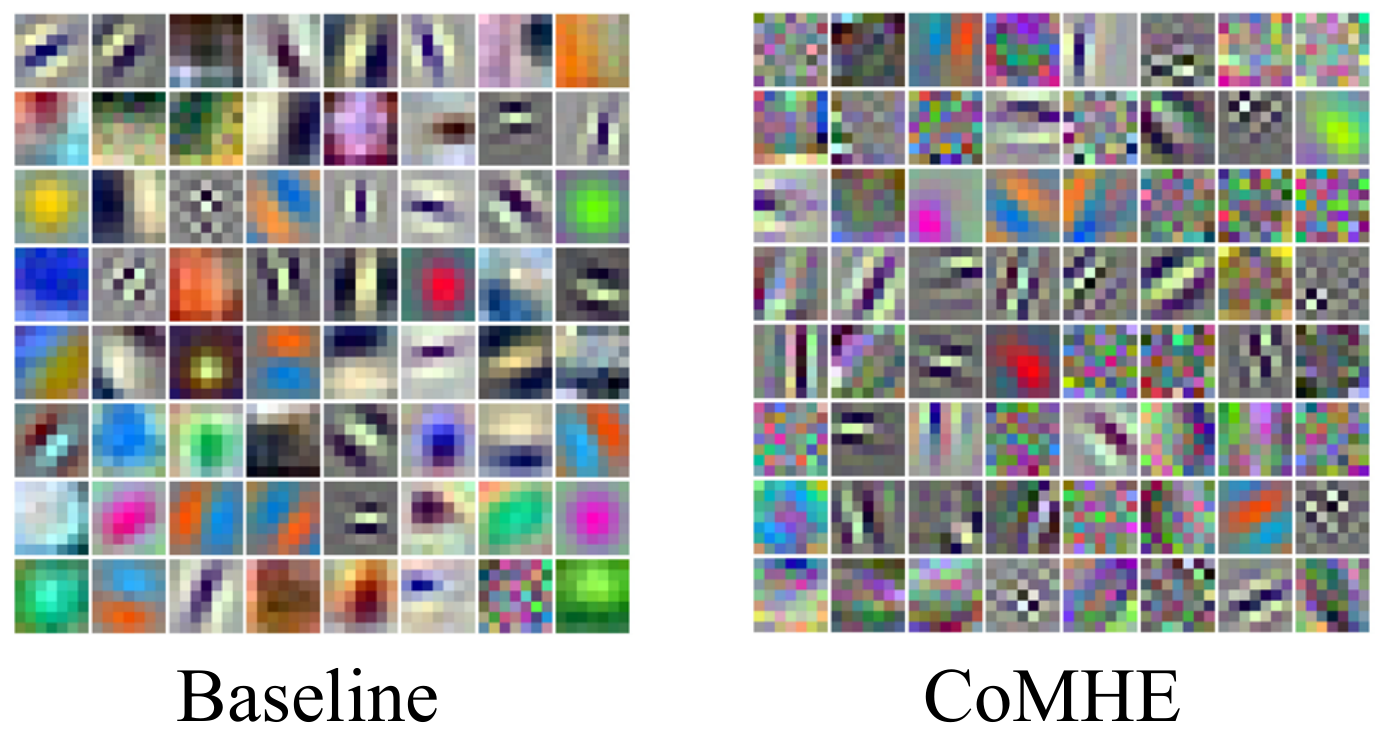}
    \vspace{-0.16in} 
  \caption{\scriptsize Visualized first-layer filters.}\label{filter}
    \vspace{-0.235in} 
  \end{center}
\end{wrapfigure}

\par
Besides the accuracy improvement, we also visualize in Fig.~\ref{filter} the 64 filters in the first-layer learned by the baseline ResNet and the CoMHE-regularized ResNet. The filters look quite different after we regularize the network using CoMHE. Each filter learned by baseline focuses on a particular local pattern (\eg, edge, color and shape) and each one has a clear local semantic meaning. In contrast, filters learned by CoMHE focuses more on edges, textures and global patterns which do not necessarily have a clear local semantic meaning. However, from a representation basis perspective, having such global patterns may be beneficial to the recognition accuracy. We also observe that filters learned by CoMHE pay less attention to color.

\setlength{\columnsep}{9pt}
\begin{wraptable}{r}[0cm]{0pt}
	\centering
	\scriptsize
	\newcommand{\tabincell}[2]{\begin{tabular}{@{}#1@{}}#2\end{tabular}}
	\setlength{\tabcolsep}{4pt}
\begin{tabular}{c || c c c} 
\specialrule{0em}{-11pt}{0pt}
  \hline
  \multirow{1}{*}{Method} & PN & PN (T) & PN++ \\
 \hline\hline
 Original  & 87.1 & 89.20 & 90.07 \\
 MHE~\cite{LiuNIPS18}    & 87.31 & 89.33 & 90.25 \\
 HS-MHE~\cite{LiuNIPS18}    & 87.44 & 89.41 & 90.31 \\
 RP-CoMHE   & 87.82 & 89.69 & 90.52 \\
 AP-CoMHE   & \textbf{87.85} & \textbf{89.70} & \textbf{90.56}\\
 \hline
 \specialrule{0em}{0pt}{-9pt}
\end{tabular}
\caption{\scriptsize Accuracy (\%) on ModelNet-40.}
\label{pointnet}
\vspace{-2.5mm}
\end{wraptable}

\vspace{-0.4mm}
\subsection{Point Cloud Recognition}
\vspace{-0.8mm}
We evaluate CoMHE on point cloud recognition. Our goal is to validate the effectiveness of CoMHE on a totally different network architecture with a different form of input data structure, rather than achieving state-of-the-art performance on point cloud recognition. To this end, we conduct experiments on widely used neural networks that handles point clouds: PointNet~\cite{qi2017pointnet} (PN) and PointNet++~\cite{qi2017pointnet++} (PN++). We combine half-space MHE, RP-CoMHE and AP-CoMHE into PN (without T-Net), PN (with T-Net) and PN++. More experimental details are given in Appendix~\ref{arch}. We test the performance on ModelNet-40~\cite{wu20153d}. Specifically, since PN can be viewed as $\thickmuskip=2mu \medmuskip=2mu 1\times 1$ convolutions before the max pooling layer, we can apply all these MHE variants similarly to CNN. After the max pooling layer, there is a standard fully connected network where we can still apply the MHE variants. We compare the performance of regularizing PN and PN++ with half-space MHE, RP-CoMHE or AP-CoMHE. Table~\ref{pointnet} shows that all MHE variants consistently improve PN and PN++, while RP-CoMHE and AP-CoMHE again perform the best among all. We demonstrate that CoMHE is generally useful for different types of input data (not limit to images) and different types of neural networks. CoMHE is also useful in graph neural networks (Appendix~\ref{graph}).

\vspace{-1mm}
\section{Concluding Remarks} 
\vspace{-1mm}
Since naively minimizing hyperspherical energy yields suboptimal solutions, we propose a novel framework which projects the neurons to suitable spaces and minimizes the energy there. Experiments validate CoMHE's superiority.

\vspace{1.4mm}
{
\begin{spacing}{0.75}
{\footnotesize \noindent\textbf{Acknowledgements.} The research is partially supported by NSF BigData program under IIS-1838200, NSF CPS program under CMMI-1932187, and USDOT via C2SMART under 69A3551747124.}
\end{spacing}
}

\clearpage
{
\small
\bibliographystyle{unsrt}
\bibliography{egbib}
}

\newpage
\onecolumn
\begin{appendix}
\begin{center}
{\LARGE \textbf{Appendix}}
\end{center}

\vspace{-3mm}
\section{Experimental Details}\label{arch}
\vspace{-1mm}
\begin{table*}[h]
	\renewcommand{\captionlabelfont}{\footnotesize}
	\centering
	\setlength{\abovecaptionskip}{4pt}
	\setlength{\belowcaptionskip}{-5pt}
	\scriptsize
	\begin{tabular}{|c|c|c|c|c|c|c|}
		\hline
		Layer & CNN-6 & CNN-9 & CNN-15 \\
		\hline\hline
		Conv1.x  & [3$\times$3, 64]$\times$2 & [3$\times$3, 64]$\times$3 & [3$\times$3, 64]$\times$5 \\\hline
		Pool1&\multicolumn{3}{c|}{2$\times$2 Max Pooling, Stride 2}\\\hline
		Conv2.x  & [3$\times$3, 128]$\times$2 & [3$\times$3, 128]$\times$3 & [3$\times$3, 128]$\times$5 \\\hline
		Pool2 & \multicolumn{3}{c|}{2$\times$2 Max Pooling, Stride 2}\\\hline
		Conv3.x & [3$\times$3, 256]$\times$2 & [3$\times$3, 256]$\times$3 & [3$\times$3, 256]$\times$5 \\\hline
		Pool3 & \multicolumn{3}{c|}{2$\times$2 Max Pooling, Stride 2}\\\hline
		Fully Connected & 256 & 256 & 256  \\\hline
	\end{tabular}
	\caption{\footnotesize Our plain CNN architectures with different convolutional layers. Conv1.x, Conv2.x and Conv3.x denote convolution units that may contain multiple convolution layers. E.g., [3$\times$3, 64]$\times$3 denotes 3 cascaded convolution layers with 64 filters of size 3$\times$3.}\label{netarch}
\end{table*}

\begin{table*}[h]
	\renewcommand{\captionlabelfont}{\footnotesize}
	\newcommand{\tabincell}[2]{\begin{tabular}{@{}#1@{}}#2\end{tabular}}
	\centering
	\setlength{\abovecaptionskip}{4pt}
	\setlength{\belowcaptionskip}{5pt}
	\scriptsize
	\begin{tabular}{|c|c|c|c|c|}
		\hline
		Layer & ResNet-32 for CIFAR-10/100 & ResNet-18 for ImageNet-2012 & ResNet-34 for ImageNet-2012\\
		\hline\hline
		Conv0.x & N/A & \tabincell{c}{[7$\times$7, 64], Stride 2\\3$\times$3, Max Pooling, Stride 2} & \tabincell{c}{[7$\times$7, 64], Stride 2\\3$\times$3, Max Pooling, Stride 2}\\\hline
		Conv1.x & \tabincell{c}{[3$\times$3, 64]$\times$1\\$\left[\begin{aligned} &3\times 3, 64\\&3\times3, 64\end{aligned}\right]\times 5$} & $\left[\begin{aligned} &3\times 3, 64\\&3\times3, 64\end{aligned}\right]\times 2$ & $\left[\begin{aligned} &3\times 3, 64\\&3\times3, 64\end{aligned}\right]\times 3$\\ \hline
		Conv2.x  & \tabincell{c}{$\left[\begin{aligned} &3\times3, 128\\&3\times3, 128\end{aligned}\right]\times 5$}  & $\left[\begin{aligned} &3\times 3, 128\\&3\times3, 128\end{aligned}\right]\times 2$ & $\left[\begin{aligned} &3\times 3, 128\\&3\times3, 128\end{aligned}\right]\times 4$\\\hline
		Conv3.x  & \tabincell{c}{$\left[\begin{aligned} &3\times3, 256\\&3\times3, 256\end{aligned}\right]\times 5$} & $\left[\begin{aligned} &3\times 3, 256\\&3\times3, 256\end{aligned}\right]\times 2$ & $\left[\begin{aligned} &3\times 3, 256\\&3\times3, 256\end{aligned}\right]\times 6$ \\\hline
		Conv4.x & N/A  & $\left[\begin{aligned} &3\times 3, 512\\&3\times3, 512\end{aligned}\right]\times 2$ & $\left[\begin{aligned} &3\times 3, 512\\&3\times3, 512\end{aligned}\right]\times 3$ \\\hline
		& \multicolumn{3}{c|}{Average Pooling}  \\\hline
	\end{tabular}
	\caption{\footnotesize Our ResNet architectures with different convolutional layers. Conv0.x, Conv1.x, Conv2.x, Conv3.x and Conv4.x denote convolution units that may contain multiple convolutional layers, and residual units are shown in double-column brackets. Conv1.x, Conv2.x and Conv3.x usually operate on different size feature maps. These networks are essentially the same as \cite{he2015deep}, but some may have a different number of filters in each layer. The downsampling is performed by convolutions with a stride of 2. E.g., [3$\times$3, 64]$\times$4 denotes 4 cascaded convolution layers with 64 filters of size 3$\times$3, S2 denotes stride 2. }\label{netarch2}
\end{table*}

\textbf{Image recognition settings.} The network architectures used in the paper are elaborated in Table~\ref{netarch} and Table~\ref{netarch2}. For CIFAR-10 and CIFAR-100, we use batch size 128. We start with learning rate 0.1, divide it when the performance is saturated. For ImageNet-2012, we use batch size 64 and start with learning rate 0.1. The learning rate is divided by 10 when the performance is saturated, and the training is terminated at 500k iterations. For ResNet-50 on ImageNet-2012, we use exactly the same architecture as \cite{he2015deep}. Note that, for all the compared methods, we always use the best possible hyperparameters to make sure that the comparison is fair. The baseline has exactly the same architecture and training settings as the one that CoMHE uses. For both half-space MHE and all the variants of CoMHE in hidden layers, we set the weighting hyperparameter as $1$ in all experiments. \cite{LiuNIPS18} already shows that MHE type of losses are not senstitive to the weighting hyperparameter.
We use $\thickmuskip=2mu \medmuskip=2mu 1e-5$ for the orthonormal and orthogonal regularizations (the hyperparameter is obtained via cross-validation). If not otherwise specified, standard $\ell_2$ weight decay ($\thickmuskip=2mu \medmuskip=2mu 1e-4$) is applied to all the neural network including baselines and the networks that use MHE regularization. 
Note that, all the neuron weights in the neural networks used in the paper are not normalized (unless otherwise specified), but both MHE and CoMHE will normalize the neuron weights while computing the regularization loss. For all experiments, we use $s=2$ in both MHE and CoMHE. As a result, \emph{CoMHE does not need to modify any component of the original neural networks, and it can simply be viewed as an extra regularization loss that can boost the performance}. All the network architectures are implemented by ourselves, so there might be performance difference between our implementation and the original paper due to some different network and optimization hyperparameters. For example, due to the limitation of computation resources, our batch size for ImageNet-2012 is 64 batch size, which might has some performance loss. However, the network and training settings for the baseline and all the compared regularizations are the same, which ensures the fairness of our experiments. In ImageNet classification, we emphasize that our purpose is to compare the gain brought by different regularizations rather than achieving the state-of-the-art performance. We implement the data augmentation in the ImageNet experiment, following AlexNet~\cite{krizhevsky2012imagenet} and SphereNet~\cite{liu2017hyper}, so the accuracy may be lower than using the more complicated data augmentation used in original ResNet~\cite{he2015deep}.
\par
\textbf{Point cloud recognition settings.} For all the PointNet and PointNet++ experiments, we exactly follow the same setting in the original papers~\cite{qi2017pointnet,qi2017pointnet++} and their official repositories\footnote{\url{https://github.com/charlesq34/pointnet}} \footnote{\url{https://github.com/charlesq34/pointnet2}}. Specifically, we combine CoMHE regularization to neurons in all the $\thickmuskip=2mu \medmuskip=2mu 1\times1$ convolution layers before the max pooling layer and the multi-layer perceptron classifier after the max pooling layer. All the regularization is added without changing any components in PointNet. For PointNet experiments, we use point number 1024, batch size 32 and Adam optimizer started with learning rate 0.001, the learning rate will decay by 0.7 every 200k iterations, and the training is terminated at 250 epochs. For PointNet++ experiments, since the MRG (multi-resolution grouping) model is not provided in the official repository, we use the SSG (single scale grouping) model as baseline. Specifically, we use point number 1024, batch size 16 and Adam optimizer started with learning rate 0.001, the learning rate will decay by 0.7 every 200k iterations, and the training is terminated at 251 epochs. For all experiments, we use $s=1$ in both MHE and CoMHE.
\par
We evaluate on PointNet with T-Net and without T-Net in order to demonstrate that CoMHE is not sensitive to architecture modifications. We follow all the default hyperparameters used in the official released code, and the only difference is that we further combine an additional regularization loss for the neurons in each layer. One can observe that CoMHE consistently performs better than half-space MHE~\cite{LiuNIPS18}.
\par
Besides PointNet, we combine CoMHE to PointNet++~\cite{qi2017pointnet++} and further show the improvement of generalization introduced by CoMHE is agnostic to the architecture. We evaluate PointNet++ with and without CoMHE on ModelNet-40. Note that, we exactly follow the released code in the official repository where PointNet++ uses the single scale grouping model. Because the original paper~\cite{qi2017pointnet++} uses the multi-resolution grouping model, the baseline performance reported in our paper is not as good as the accuracy reported in the original paper. However, our purpose is to validate the effectiveness of CoMHE, so we only focus on the performance gain. One can observe that CoMHE achieves about 0.5\% accuracy gain, while half-space MHE~\cite{LiuNIPS18} only has about 0.2\% accuracy gain.

\newpage
\section{Experimental Details and Full Results of Different Hyperspherical Minimization Strategies}\label{exp1_detail}

\subsection{General experimental details} 

\begin{table*}[h]
	\renewcommand{\captionlabelfont}{\footnotesize}
	\centering
	\setlength{\abovecaptionskip}{5pt}
	\setlength{\belowcaptionskip}{0pt}
	\scriptsize
	\begin{tabular}{|c|c|c|}
		\hline
		Layer & CNN-3 & CNN-9 \\
		\hline\hline
		Conv1.x  & [3$\times$3, 32]$\times$1 & [3$\times$3, 32]$\times$3 \\\hline
		Pool1&\multicolumn{2}{c|}{2$\times$2 Max Pooling, Stride 2}\\\hline
		Conv2.x  & [3$\times$3, 64]$\times$1 & [3$\times$3, 64]$\times$3 \\\hline
		Pool2 & \multicolumn{2}{c|}{2$\times$2 Max Pooling, Stride 2}\\\hline
		Conv3.x & [3$\times$3, 64]$\times$1 & [3$\times$3, 64]$\times$3 \\\hline
		Pool3 & \multicolumn{2}{c|}{2$\times$2 Max Pooling, Stride 2}\\\hline
		Fully Connected & 64 & 64  \\\hline
	\end{tabular}
	\caption{\footnotesize Our small plain CNN architectures with different convolutional layers for the illustrative experiment in Fig.~\ref{rotation}. Conv1.x, Conv2.x and Conv3.x denote convolution units that may contain multiple convolution layers. E.g., [3$\times$3, 64]$\times$3 denotes 3 cascaded convolution layers with 64 filters of size 3$\times$3.}\label{small_netarch}
\end{table*}

The training details are the same as the CIFAR-100 experiment described in Appendix~\ref{arch}, except that we use different network structure here. All the optimizations of CNNs use the Stochastic gradient descent with momentum $0.9$. The number of iteration and learning rate decay exactly follows the CIFAR-100 experiment in Section~\ref{cifar_explore}. The results in Fig.~\ref{rotation} are obtained with the CNN-9 described in Table~\ref{small_netarch}. In order to show that the conclusion obtained using CNN-9 is architecture-agnostic, we also conduct the same experiment on CNN-3. The width of CNN-3 and CNN-9 is smaller than the architectures we used in Section~\ref{cifar_explore}, because the orthogonal training consumes more GPU memory when the size of convolution kernels gets larger. Since the width of CNNs is still the same for all the compared regularizations, it will not affect the validity of our conclusions. For standard, MHE and CoMHE training, the weight decay will be used by default. For rotation training, the weight decay is no longer needed since it does not need to learn the weights for neurons. Note that, all networks (with different regularization) are initialized with the same weights and therefore \textbf{have the same hyperspherical energy at the beginning}.

\subsection{Details of different training strategies}
\textbf{Standard Training}. In standard training, we use the conventional end-to-end training for all the neurons in the CNN via back-propagation. The standard training is the same as the way that baselines are trained in Section~\ref{expandresult}.
\par

\begin{figure}[h]
  \centering
  \renewcommand{\captionlabelfont}{\footnotesize}
  \setlength{\abovecaptionskip}{3pt}
  \setlength{\belowcaptionskip}{0pt}
\includegraphics[width=4.4in]{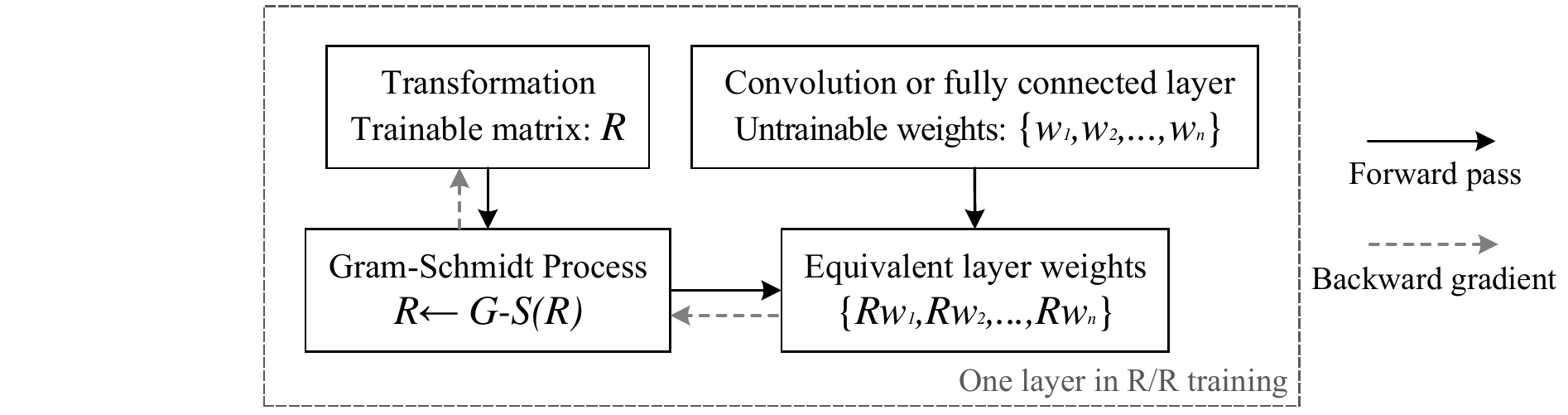}
  \caption{\footnotesize Illustration of one layer in rotation/reflection training.}\label{rr_app}
\end{figure}

\textbf{Rotation Training}. We introduce an interesting baseline here for better comparison to our CoMHE. The core of rotation training~\cite{OPT2020} is to learn an orthogonal matrix for the neurons in the same layer with the weights of these neurons being fixed. Such an orthogonal matrix is learned individually in every layer except the classifier layer (\ie, the final layer that outputs the class logic). The classifier layer is still learned from scratch via back-propagation. Rotation training is a special case of the orthogonal over-parameterized training in \cite{OPT2020}. Specifically, we denote $N$ neurons in the $i$-th (convolution or fully-connected) layer as $\{\bm{w}_1^{(i)},\bm{w}_2^{(i)},\cdots,\bm{w}_N^{(i)}\}\in\mathbb{R}^d$. After randomly initializing these neuron weights with the method in \cite{he2015delving}, we will fix $\{\bm{w}_1^{(i)},\bm{w}_2^{(i)},\cdots,\bm{w}_N^{(i)}\}$ and make them \textbf{untrainable} during the entire training procedure. We will learn an orthogonal matrix $\bm{R}^{(i)}\in\mathbb{R}^{d\times d}$ for the neurons in the $i$-th layer such that the equivalent neurons become $\{\bm{R}^{(i)}\bm{w}_1^{(i)},\bm{R}^{(i)}\bm{w}_2^{(i)},\cdots,\bm{R}^{(i)}\bm{w}_N^{(i)}\}$. The angle between the $j$-th neuron and $k$-th neuron is preserved after the rotation/reflection, because we have
\begin{equation}
\small
\cos(\theta_{(\bm{R}\bm{w}_j,\bm{R}\bm{w}_k)})=\frac{(\bm{R}\bm{w}_j)^\top \bm{R}\bm{w}_k}{\|\bm{R}\bm{w}_j\|\cdot\|\bm{R}\bm{w}_k\|}=\frac{\bm{w}_j^\top\bm{w}_k}{\|\bm{w}_j\|\cdot\|\bm{w}_k\|}=\cos(\theta_{(\bm{w}_j,\bm{w}_k)}).
\end{equation}
Therefore, rotation training only learns the orthogonal matrices $\{\bm{R}^{(1)},\bm{R}^{(2)},\cdots,\bm{R}^{(L)}\}$ for neurons of all the convolution and fully-connected layers except the classifier layer. This training strategy will always keep the hyperspherical energy the same during the training process.
\par
For the implementation of rotation training, we use the Gram-Schmidt process to orthonormalizing all the learnable matrices $\bm{R}^{(i)},\forall i$ before multiplying them to the neuron weights. Since the Gram-Schmidt process is differentiable, we can directly insert it to process the learnable matrices $\bm{R}^{(i)},\forall i$ in the forward pass. We show an overview of forward and backward pass in one layer of rotation training in Fig.~\ref{rr_app} to demonstrate the procedure how we orthonormalize the matrices $\bm{R}^{(i)},\forall i$ and apply them to the fixed neuron weights. 

\par
\textbf{MHE Training}. MHE training is to train the neurons with both data fitting loss and MHE regularization loss from scratch, following the same procedure in \cite{LiuNIPS18}.
\par
\textbf{CoMHE Training} Similar to MHE training, CoMHE training is to train the neurons with both data fitting loss and CoMHE regularization loss (including RP-CoMHE and AP-CoMHE) from scratch. Details are given in the main paper.

\subsection{Experimental results}
\textbf{Experiments on CIFAR-100 with CNN-9 (BatchNorm) (also shown in Fig.~\ref{rotation})}. We first conduct experiments on CIFAR-100 with CNN-9 as the backbone architecture. For all the compared methods, we use batch normalization. Note that, the experimental results here are the extended results of Fig.~\ref{rotation}. We compute the hyperspherical energy of $N$ neurons using the following definition of half-space hyperspherical energy ($s=1$)~\cite{LiuNIPS18} (the same as the regularization loss):
\begin{equation}\label{hs_energy_loss}
\small
\begin{aligned}
    \bm{E} = \frac{1}{2N(2N-1)}\sum_{i=1}^{2N}\sum_{j=1,j\neq i}^{2N}  \frac{1}{\norm{\hat{\bm{w}}_i-\hat{\bm{w}}_j}}
\end{aligned}
\end{equation}
where $\hat{\bm{w}}_{N+i}=-\hat{\bm{w}}_i,\ 0\leq i \leq N$. This is the hyperspherical energy of neurons in one layer. The total hyperspherical energy needs to sum up the energy from all the layers. We show the hyperspherical energy and the accuracy v.s. iteration in Fig.~\ref{c100cnn9bn}.

The results in Fig.~\ref{c100cnn9bn} shows that rotation training can largely improve the accuracy compared to the baseline, indicating that the hyperspherical energy can characterize the generalization and lower hyperspherical energy leads to better generalization. The rotation training is able to perform similarly to HS-MHE, showing the advantage of low hyperspherical energy. It also shows that the performance of the original MHE (\ie, HS-MHE) can be achieved by a simple rotation/reflection training strategy.
\par
\vspace{2mm}

\textbf{Experiments on CIFAR-100 with CNN-3 (BatchNorm)}. To show that the same behavior will also happen in different network structure, we conduct experiments on CIFAR-100 with a 3-layer CNN as shown in Table~\ref{small_netarch}. Batch normalization is also used. The results in Fig.~\ref{c100cnn3bn} confirm that rotation training performs better than the baseline and MHE but still worse than our proposed CoMHE. The results verify that hyperspherical energy is an important generaliability indicator for trained networks.

\par
\vspace{2mm}
\textbf{Experiments on CIFAR-100 with CNN-9 (no BatchNorm)}. To show that batch normalization does not affect our conclusion, we also conduct experiments on CIFAR-100 without using batch normalization. We use the CNN-9 architecture as shown in Table~\ref{small_netarch}. The results in Fig.~\ref{c100cnn9nobn} show that our conclusion holds even without batch normalization. Hyperspherical energy plays an important role in the network generalization. In general, lower hyperspherical energy leads to better generalization, even though such low hyperspherical energy (in rotation training) is achieved by zero-mean Gaussian initialization.
\par
\vspace{2mm}
\textbf{Experiments on CIFAR-10 with CNN-9 (BatchNorm)}. To show that the same empirical behavior is consistent in different dataset, we further conduct experiments on CIFAR-10. We use CNN-9 with batch normalization as our backbone architecture. The results in Fig.~\ref{c10cnn9bn} show that rotation training still performs much better than the baseline and slightly worse than CoMHE. Most interestingly, rotation training can even perform better than the original MHE.

\subsection{Conclusion and Discussion}
We have extensively tested the hyperspherical energy and accuracy of standard, rotation/reflection, MHE, and CoMHE training under multiple circumstances. The empirical evidences consistently show that hyperspherical energy is of great importance and is able to indicate the potential generalizability of a trained network. Even if we use randomly initialized neurons with low hyperspherical energy, we can still have impressive performance (better than MHE) if proper rotations/reflections of these neurons are learned. Note that rotation/reflection will not change the hyperspherical energy. Therefore, how to effectively minimize the hyperspherical energy is of great significance and is also the central focus of CoMHE.

\begin{figure}[t]
  \centering
  \renewcommand{\captionlabelfont}{\footnotesize}
  \setlength{\abovecaptionskip}{3pt}
  \setlength{\belowcaptionskip}{0pt}
\includegraphics[width=6.75in]{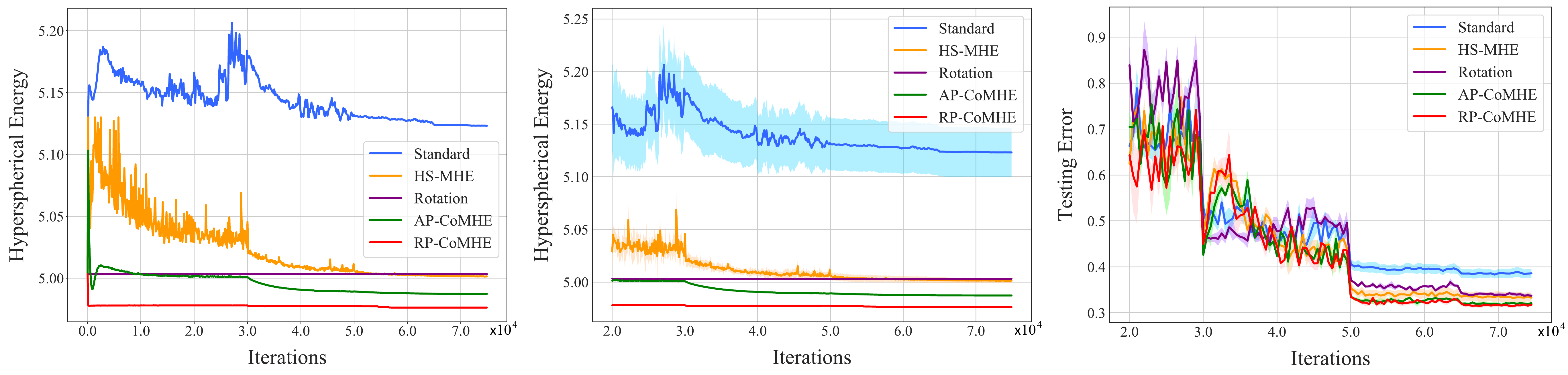}
  \caption{\footnotesize Results on CIFAR-100 with CNN-9 (BatchNorm). Left: hyperspherical energy v.s. iteration during the entire training. Middle: hyperspherical energy v.s. iteration after the 20000-th iterations (with standard deviation). Right: Testing Error on CIFAR-100 (with standard deviation).}\label{c100cnn9bn}
\end{figure}

\begin{figure}[t]
  \centering
  \renewcommand{\captionlabelfont}{\footnotesize}
  \setlength{\abovecaptionskip}{3pt}
  \setlength{\belowcaptionskip}{0pt}
\includegraphics[width=6.75in]{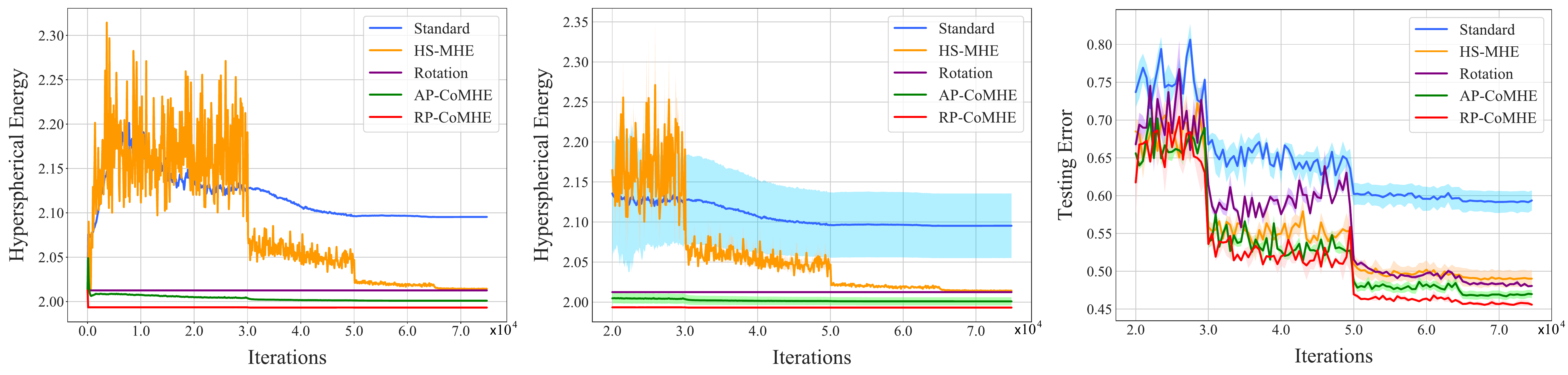}
  \caption{\footnotesize Results on CIFAR-100 with CNN-3 (BatchNorm). Left: hyperspherical energy v.s. iteration during the entire training. Middle: hyperspherical energy v.s. iteration after the 20000-th iterations (with standard deviation). Right: Testing Error on CIFAR-100 (with standard deviation).}\label{c100cnn3bn}
\end{figure}

\begin{figure}[t]
  \centering
  \renewcommand{\captionlabelfont}{\footnotesize}
  \setlength{\abovecaptionskip}{3pt}
  \setlength{\belowcaptionskip}{0pt}
\includegraphics[width=6.75in]{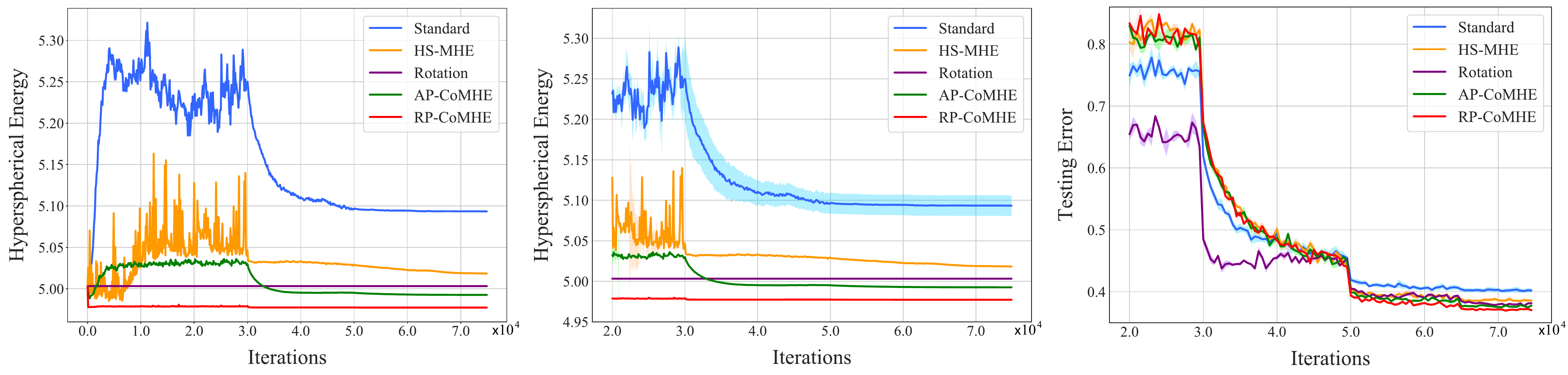}
  \caption{\footnotesize Results on CIFAR-100 with CNN-9 (no BatchNorm is applied). Left: hyperspherical energy v.s. iteration during the entire training. Middle: hyperspherical energy v.s. iteration after the 20000-th iterations (with standard deviation). Right: Testing Error on CIFAR-100 (with standard deviation).}\label{c100cnn9nobn}
\end{figure}

\begin{figure}[t]
  \centering
  \renewcommand{\captionlabelfont}{\footnotesize}
  \setlength{\abovecaptionskip}{3pt}
  \setlength{\belowcaptionskip}{0pt}
\includegraphics[width=6.75in]{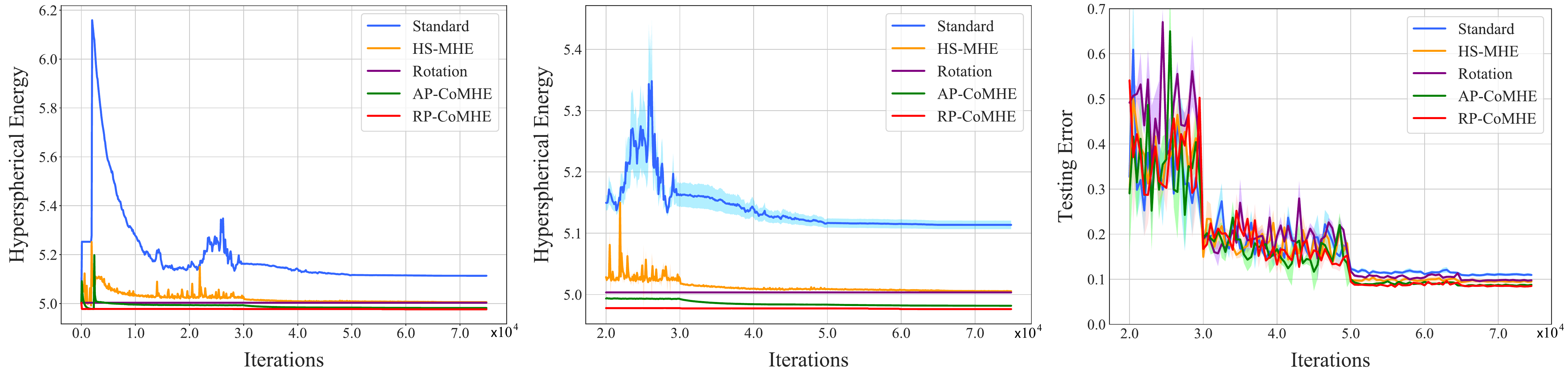}
  \caption{\footnotesize Results on CIFAR-10 with CNN-9 (BatchNorm). Left: hyperspherical energy v.s. iteration during the entire training. Middle: hyperspherical energy v.s. iteration after the 20000-th iterations (with standard deviation). Right: Testing Error on CIFAR-10 (with standard deviation).}\label{c10cnn9bn}
\end{figure}

\clearpage
\newpage

\section{Two Random Vectors Are Approximately Orthogonal in High Dimensions}\label{hd_orth}
We have two random uniform vectors $\frac{\bm{X}}{\|\bm{X}\|}$ and $\frac{\bm{Y}}{\|\bm{Y}\|}$ where $\bm{X}$ and $\bm{Y}$ are normal distributions. Then the inner product of these two independent unit vectors is $\frac{\langle\bm{X},\bm{Y}\rangle}{\|\bm{X}\|\|\bm{Y}\|}$. When $n\rightarrow +\infty$, according to the law of large numbers, we have that $\frac{\bm{X}}{\sqrt{n}}\rightarrow 1$ almost surely. By the central limit theorem, $\frac{\langle\bm{X},\bm{Y}\rangle}{\sqrt{n}}$ converges in distribution to a standard one-dimensional normal distribution. Therefore, we have in distribution that
\begin{equation}
    \sqrt{n}\cdot\langle\bm{U},\bm{V}\rangle\rightarrow z
\end{equation}
where $\bm{U}=\frac{\bm{X}}{\|\bm{X}\|}$, $\bm{V}=\frac{\bm{Y}}{\|\bm{Y}\|}$ and $z$ follows a normal distribution. Then for every $\epsilon>0$, we have that 
\begin{equation}
P(|\langle\bm{U},\bm{V}\rangle|)\rightarrow0
\end{equation}
which implies that the probability that $\bm{U}$ and $\bm{V}$ are nearly orthogonal approaches to $1$ when $n\rightarrow +\infty$. Similarly, we can also conclude that $k$ independent uniform unit vectors on the hypersphere are nearly orthogonal with very high probability when the dimension becomes higher.

\section{Johnson-Lindenstrauss Lemma}\label{app_jll}
\begin{lemma}[Johnson-Lindenstrauss Lemma~\cite{dasgupta2003elementary,kaban2015improved}]\label{jll}
Let $w_1,w_2\in\mathbb{R}^d$ be vectors, and $\bm{P}\in\mathcal{R}^{k\times d},k<d$ be a random projection matrix with entries i.i.d. drawn from a $0$-mean $\sigma$-subgaussian distribution. With $\bm{P}\bm{w}_1,\bm{P}\bm{w}_2\in\mathbb{R}^k$ being the projected vectors of $\bm{w}_1,\bm{w}_2$, then, $\forall \epsilon\in(0,1)$,
\begin{equation}
\begin{aligned}
(1-\epsilon)\norm{\bm{w}_1-\bm{w}_2}^2k\sigma^2&<\norm{\bm{P}\bm{w}_1-\bm{P}\bm{w}_2}^2<(1+\epsilon)\norm{\bm{w}_1-\bm{w}_2}^2k\sigma^2
\end{aligned}
\end{equation}
holds with probability at least $1-2\exp({-\frac{k\epsilon^2}{8}})$.
\end{lemma}
\section{Proofs}
In the section, we aim to provide the complete proof for self-containedness.
\subsection{Lemma~\ref{lem1}}
We take the expectation of the inner product between projected vectors:

\begin{equation}
    \begin{aligned}
        \mathbb{E}(\langle \bm{P}\bm{w}_1,\bm{P}\bm{w}_2 \rangle) &= \frac{1}{n}\mathbb{E}\bigg( \sum_{l=1}^n \big(\sum_{j=1}^d r_{lj}\{\bm{w}_1\}_j \sum_{i=1}^d r_{li} \{\bm{w}_2\}_i \big) \bigg)\\
        &=\frac{1}{n}\sum_{l=1}^n \bigg( \sum_{j=1}^d \mathbb{E}(r_{lj}^2)\{\bm{w}_1\}_j\{\bm{w}_2\}_j +\sum_{j=1}^d \mathbb{E}(r_{lj})\{\bm{w}_1\}_j \cdot\sum_{i\neq j: i=1}^d\mathbb{E}(r_{li})\{\bm{w}_2\}_i \bigg)\\
        &=\langle \bm{w}_1,\bm{w}_2 \rangle
    \end{aligned}
\end{equation}
where $\{\bm{w}_1\}_i$ is the $i$-th element of the vector $\bm{w}_1$, and $\{\bm{w}_2\}_i$ is the $i$-th element of the vector $\bm{w}_2$. From the equation, we see that the lemma is proved.\qed

\subsection{Theorem~\ref{thm1}}
Before proving the the main theorem, we first show a lemma from \cite{kaban2015improved}.
\begin{lemma}[Dot Product under Random Projection]
Let $\bm{w}_1,\bm{w}_2\in\mathbb{R}^d$, $\bm{P}\in\mathbb{R}^{k\times d},k<d$ be a random projection matrix having i.i.d. 0-mean subgaussian entries with parameter $\sigma^2$, and $\bm{P}\bm{w}_1,\bm{P}\bm{w}_2$ be the images of $\bm{w}_1,\bm{w}_2$ under projection $\bm{P}$. Then, $\forall \epsilon\in(0,1)$:
\begin{equation}\label{new1}
    \bm{w}_1^\top\bm{w}_2 k \sigma^2-\epsilon k\sigma^2\norm{\bm{w}_1}\norm{\bm{w}_2}<(\bm{P}\bm{w}_1)^\top\bm{P}\bm{w}_2<\bm{w}_1^\top\bm{w}_2 k \sigma^2+\epsilon k\sigma^2\norm{\bm{w}_1}\norm{\bm{w}_2}
\end{equation}
holds with probability $1-2\exp(-\frac{k\sigma^2}{8})$.
\label{dotprod}
\end{lemma}

From Lemma~\ref{jll}, we have that 

\begin{equation}\label{new2}
\begin{aligned}
&(1-\epsilon)\norm{\bm{w}_1}^2k\sigma^2<\norm{\bm{P}\bm{w}_1}^2<(1+\epsilon)\norm{\bm{w}_1}^2k\sigma^2\\
&(1-\epsilon)\norm{\bm{w}_2}^2k\sigma^2<\norm{\bm{P}\bm{w}_2}^2<(1+\epsilon)\norm{\bm{w}_2}^2k\sigma^2
\end{aligned}
\end{equation}
which holds with probability $\big(1-2\exp({-\frac{k\epsilon^2}{8}})\big)^2$.
\par
Then we combine Eq.~\ref{new2} to Lemma~\ref{dotprod} and obtain that
\begin{equation}
\begin{aligned}
    \frac{\cos(\theta_{(\bm{w}_1,\bm{w}_2)})-\epsilon}{1+\epsilon}<\cos(\theta_{(\bm{P}\bm{w}_1,\bm{P}\bm{w}_2)})
    <\frac{\cos(\theta_{(\bm{w}_1,\bm{w}_2)})+\epsilon}{1-\epsilon}
\end{aligned}
\end{equation}
which holds with probability $\big(1-2\exp({-\frac{k\epsilon^2}{8}})\big)^2$. $\theta_{(\bm{P}\bm{w}_1,\bm{P}\bm{w}_2)}$ denotes the angle between $\bm{P}\bm{w}_1$ and $\bm{P}\bm{w}_2$, and $\theta_{(\bm{w}_1,\bm{w}_2)}$ denotes the angle between $\bm{w}_1$ and $\bm{w}_2$. \qed

\subsection{Theorem~\ref{thm2}}
Before proving our main theorem, we first show a lemma~\cite{shi2012margin} below:
\begin{lemma}\label{long}
For any $w\in\mathbb{R}^d$, any random Gaussian matrix $\bm{P}\in\mathbb{R}^{k\times d}$ where $\bm{P}_{ij}=\frac{1}{\sqrt{n}}r_{ij}$ and $r_{ij},\forall i,j$ are i.i.d. random variables from $\mathcal{N}(0,1)$, and $\epsilon\in(0,1)$
\begin{equation}
    \textnormal{Pr}\bigg( (1-\epsilon)\leq \frac{\norm{\bm{P}\bm{w}}^2}{\norm{\bm{w}}^2} \leq(1+\epsilon) \bigg)\geq 1-2\exp\big(-\frac{n}{2}(\frac{\epsilon^2}{2}-\frac{\epsilon^3}{3})\big) 
\end{equation}
\end{lemma}
\begin{proof}[Proof of Lemma~\ref{long}]
From Lemma~\ref{lem1}, we have that $\mathbb{E}(\norm{\bm{P}\bm{w}}^2)=\norm{\bm{w}}^2$. Due to 2-stability of the Gaussian distribution, we have that $\sum_{j=1}^d r_{lj}w_j=\norm{\bm{w}}z_l$ where $z_l\sim \mathcal{N}(0,1)$. As a result, we have that
\begin{equation}
    \norm{\bm{P}\bm{w}}^2 =\frac{1}{n}\bm{w}^2\sum_{l=1}^n z_l^2
\end{equation}
where $\sum_{l=1}^n z_l^2$ is chi-square distributed with $n$-degree freedom. Then we apply the standard tail bound of the chi-square distribution and obtain
\begin{equation}
\begin{aligned}
    \textnormal{Pr}\bigg( \norm{\bm{P}\bm{w}}^2\leq(1-\epsilon)\norm{\bm{w}^2}\bigg)&\leq \exp\bigg( \frac{n}{2}\big(1-(1-\epsilon)+\ln(1-\epsilon)\big) \bigg)\\
    &\leq \exp (-\frac{n}{4}\epsilon^2)
\end{aligned}
\end{equation}
where the inequality $\ln(1-\epsilon)\leq - \epsilon-\frac{\epsilon^2}{2}$ is applied. Similarly, one can have
\begin{equation}
\begin{aligned}
    \textnormal{Pr}\bigg( \norm{\bm{P}\bm{w}}^2\leq(1+\epsilon)\norm{\bm{w}^2}\bigg)&\leq \exp\bigg( \frac{n}{2}\big(1-(1+\epsilon)+\ln(1+\epsilon)\big) \bigg)\\
    &\leq \exp (-\frac{n}{2}(\frac{\epsilon^2}{2}-\frac{\epsilon^3}{3}))
\end{aligned}
\end{equation}
where the inequality $\ln(1+\epsilon)\leq\epsilon-\frac{\epsilon^2}{2}+\frac{\epsilon^3}{3}$ is used.
\end{proof}
From the lemma above, we apply the union bound and have that 

\begin{equation}\label{supp1}
\begin{aligned}
&(1-\epsilon) \leq \frac{\norm{\bm{P}\bm{w}_1}^2}{\norm{\bm{w}_1}^2} \leq (1+\epsilon)\\
&(1-\epsilon) \leq \frac{\norm{\bm{P}\bm{w}_2}^2}{\norm{\bm{w}_2}^2} \leq (1+\epsilon)\\
\end{aligned}
\end{equation}
which holds with probability at least $1-4\exp(-\frac{n}{2}(\frac{\epsilon^2}{2}-\frac{\epsilon^3}{3}))$. Using Eq.~\ref{supp1}, we can have that

\begin{equation}\label{supp2}
\norm{\frac{\bm{P}\bm{w}_1}{\norm{\bm{P}\bm{w}_1}}-\frac{\bm{P}\bm{w}_2}{\norm{\bm{P}\bm{w}_2}}}^2 \leq \norm{\frac{\bm{P}\bm{w}_1}{\sqrt{1-\epsilon}\norm{\bm{w}_1}}-\frac{\bm{P}\bm{w}_2}{\sqrt{1-\epsilon}\norm{\bm{w}_2}}}^2
\end{equation}
\par
From Eq.~\ref{supp1} and the condition that $\bm{w}_1^\top\bm{w}_2>0$, we further have that

\begin{equation}\label{supp3}
\begin{aligned}
    \norm{\frac{\bm{P}\bm{w}_1}{\norm{\bm{w}_1}}-\frac{\bm{P}\bm{w}_2}{\norm{\bm{w}_2}}}^2&\leq \norm{\sqrt{1+\epsilon}-\sqrt{1-\epsilon}}^2\\
    &\leq  \norm{\sqrt{1+\epsilon}\bigg(\frac{\bm{P}\bm{w}_1}{\norm{\bm{P}\bm{w}_1}}-\frac{\bm{P}\bm{w}_2}{\norm{\bm{P}\bm{w}_2}}\bigg)}^2 + \norm{\sqrt{1+\epsilon}-\sqrt{1-\epsilon}}^2
\end{aligned}
\end{equation}
\par
Then we apply Lemma~\ref{long} to the vector $(\frac{\bm{w}_1}{\norm{\bm{w}_1}}-\frac{\bm{w}_2}{\norm{\bm{w}_2}})$ and see that

\begin{equation}\label{supp4}
\begin{aligned}
    (1-\epsilon)\norm{\frac{\bm{w}_1}{\norm{\bm{w}_1}}-\frac{\bm{w}_2}{\norm{\bm{w}_2}}}^2\leq \norm{\frac{\bm{P}\bm{w}_1}{\norm{\bm{w}_1}}-\frac{\bm{P}\bm{w}_2}{\norm{\bm{w}_2}}}^2\leq (1+\epsilon)\norm{\frac{\bm{w}_1}{\norm{\bm{w}_1}}-\frac{\bm{w}_2}{\norm{\bm{w}_2}}}^2
\end{aligned}
\end{equation}
which holds with probability $1-2\exp\big(-\frac{n}{2}(\frac{\epsilon^2}{2}-\frac{\epsilon^3}{3})\big)$. Then we have that
\begin{equation}\label{supp5}
\begin{aligned}
    \frac{\langle\bm{w}_1,\bm{w}_2\rangle}{\norm{\bm{w}_1}\norm{\bm{w}_2}}&=1-\frac{1}{2}\norm{\frac{\bm{w}_1}{\norm{\bm{w}_1}}-\frac{\bm{w}_2}{\norm{\bm{w}_2}}}^2,\\
    \frac{\langle\bm{P}\bm{w}_1,\bm{P}\bm{w}_2\rangle}{\norm{\bm{P}\bm{w}_1}\norm{\bm{P}\bm{w}_2}}&=1-\frac{1}{2}\norm{\frac{\bm{P}\bm{w}_1}{\norm{\bm{P}\bm{w}_1}}-\frac{\bm{P}\bm{w}_2}{\norm{\bm{P}\bm{w}_2}}}^2.
\end{aligned}
\end{equation}
From Eq.~\ref{supp2}, Eq.~\ref{supp3} and Eq.~\ref{supp4}, we can learn that $\norm{\frac{\bm{P}\bm{w}_1}{\norm{\bm{P}\bm{w}_1}}-\frac{\bm{P}\bm{w}_2}{\norm{\bm{P}\bm{w}_2}}}^2$ is bounded below and above. Further combining Eq.~\ref{supp5}, we have that
\begin{equation}
\begin{aligned}
    \frac{1+\epsilon}{1-\epsilon}\cos(\theta_{(\bm{w}_1,\bm{w}_2)})-
    \frac{2\epsilon}{1-\epsilon}<\cos(\theta_{(\bm{P}\bm{w}_1,\bm{P}\bm{w}_2)})<\frac{1-\epsilon}{1+\epsilon}\cos(\theta_{(\bm{w}_1,\bm{w}_2)})+\frac{1+2\epsilon}{1+\epsilon}-\frac{\sqrt{(1-\epsilon^2)}}{1+\epsilon}
\end{aligned}
\end{equation}
where $\theta_{(\bm{P}\bm{w}_1,\bm{P}\bm{w}_2)}$ denotes the angle between $\bm{P}\bm{w}_1$ and $\bm{P}\bm{w}_2$, and $\theta_{(\bm{w}_1,\bm{w}_2)}$ denotes the angle between $\bm{w}_1$ and $\bm{w}_2$.\qed

\subsection{Theorem~\ref{thm3}}

We first introduce a lemma before proving the theorem.

\begin{lemma}\label{stat}[Direct Result from \cite{cuesta2007sharp}]
Let $\mathcal{H}$ be a separable Hilbert space, and let $\mu$ be a non-degenerate Gaussian measure on $\mathcal{H}$. Let $P,Q$ be Borel probability measures on $\mathcal{H}$. Assume that:
\begin{itemize}
    \item The abosolute moments $m_n:=\int\norm{x}^ndP(x)$ are finite and satisfy $\sum_{n\geq 1} m_n^{\frac{-1}{n}}=\infty$;
    \item The set $\varepsilon(P,Q):=\{x\in\mathcal{H}:P_{\langle x\rangle}=Q_{\langle x\rangle}\}$, where ${\langle x\rangle}$ denotes the one-dimensional subspace spanned by $x$, is of positive $\mu$-measure.
\end{itemize}
Then we have $P=Q$.
\end{lemma}

If we consider $\bm{w}\in\mathbb{R}^d$ as a bounded variable, and without loss of generality, we assume that $\bm{p}=\bm{z}/\norm{\bm{z}}$ where $\bm{z}$ is a Gaussian distribution, and then the condition on the moments of $\bm{w}$ in Lemma~\ref{stat} holds. Then with the following lemma, we can easily have the desired result.

\qed

\newpage
\section{Bilateral Projection for CoMHE}\label{biproj}

In this section, we consider bilateral projection for CoMHE (BP-CoMHE) as an extension to the main paper. If we view the neurons in one layer as a matrix $\thickmuskip=2mu \medmuskip=2mu \bm{W}=\{\bm{w}_1,\cdots,\bm{w}_n\}\in\mathbb{R}^{m\times n}$ where $m$ is the dimension of neurons and $n$ is the number of neurons, then the projection mainly considered throughout the paper is to left-multiply a projection matrix $\thickmuskip=2mu \medmuskip=2mu \bm{P}_1\in\mathbb{R}^{r\times m}$ to $\bm{W}$. In fact, we can further reduce the number of neurons by right-multiplying an additional projection matrix $\thickmuskip=2mu \medmuskip=2mu \bm{P}_2\in\mathbb{R}^{n\times r}$ to $\bm{W}$. Specifically, we denote that 
\begin{equation}
    \bm{Y}_1=\bm{P}_1\bm{W}\in\mathbb{R}^{r\times n},\ \ \ \bm{Y}_2=\bm{W}\bm{P}_2\in\mathbb{R}^{m\times r}
\end{equation}
\textbf{BP-CoMHE.} The first variant of BP-CoMHE is to apply the MHE regularization separately to column vectors of $\bm{Y}_1$ and $\bm{Y}_2$, and the learning objective is given by
\begin{equation}
    \min_{\bm{W}}\bm{E}_s(\hat{\bm{y}}^{(1)}_i|_{i=1}^n)+\bm{E}_s(\hat{\bm{y}}^{(2)}_i|_{i=1}^r)
\end{equation}
where we denote that
\begin{equation}
\begin{aligned}
    \bm{Y}_1&=\{ \bm{y}^{(1)}_1,\cdots,\bm{y}^{(1)}_n \}\in\mathbb{R}^{r\times n},\\
    \bm{Y}_2&=\{ \bm{y}^{(2)}_1,\cdots,\bm{y}^{(2)}_r \}\in\mathbb{R}^{m\times r},\\
    \hat{\bm{Y}}_1&=\{ \hat{\bm{y}}^{(1)}_1=\frac{\bm{y}^{(1)}_1}{\|\bm{y}^{(1)}_1\|},\cdots,\hat{\bm{y}}^{(1)}_n=\frac{\bm{y}^{(1)}_n}{\|\bm{y}^{(1)}_n\|} \}\in\mathbb{R}^{r\times n},\\
    \hat{\bm{Y}}_2&=\{ \hat{\bm{y}}^{(2)}_1=\frac{\bm{y}^{(2)}_1}{\|\bm{y}^{(2)}_1\|},\cdots,\hat{\bm{y}}^{(2)}_r=\frac{\bm{y}^{(2)}_r}{\|\bm{y}^{(2)}_r\|} \}\in\mathbb{R}^{m\times r}.
\end{aligned}
\end{equation}
in which we have that $\hat{\bm{y}}^{(1)}_i\in\mathbb{R}^{r\times 1}$ and $\hat{\bm{y}}^{(1)}_i\in\mathbb{R}^{m\times 1}$ are two column vectors of $\hat{\bm{Y}}_1$ and $\hat{\bm{Y}}_2$, respectively. The final neurons obtained for the neural network are still $\bm{W}$. For generating the projection matrices $\bm{P}_1,\bm{P}_2$, we simply use random projection and re-initialize the random matrices every certain number of iterations. 

\par
\noindent\textbf{Low-Rank BP-CoMHE.} More interestingly, we can also approximate $\bm{W}$ with a low-rank factorization~\cite{zhou2011godec} given as follows:
\begin{equation}
\tilde{\bm{W}}= \bm{Y}_2(\bm{P}_1\bm{Y}_2)^{-1}\bm{Y}_1\in\mathbb{R}^{m\times n}
\end{equation}
which inspires us to directly use two set of parameters $\bm{Y}_1$ and $\bm{Y}_2$ to represent the equivalent neurons $\tilde{\bm{W}}$ and apply the MHE regularization separately to their column vectors (similar to the previous BP-CoMHE). Essentially, we learn the matrices $\bm{Y}_1,\bm{Y}_2$ directly via back-propagation. The projection matrix $\bm{P}_1$ is initialized as a random matrix and stays constant during the training. Different from the former case, we will not use $\bm{W}$ as the final neurons in the neural network. Instead, we will use $\tilde{\bm{W}}$ as the final neurons. The number of learnable parameters in total is $mr+nr$, which is significantly lower than the original BP-CoMHE parameterization (\ie, $mn$) if we choose $r$ to be much smaller than both $m$ and $n$. 

\newpage
\section{More Discussion on the Effectiveness of CoMHE}\label{effectiveness}
\subsection{Full results of Figure~\ref{training} in the main paper}

\begin{figure}[h]
  \centering
  \renewcommand{\captionlabelfont}{\footnotesize}
  \setlength{\abovecaptionskip}{5pt}
  \setlength{\belowcaptionskip}{3pt}
  \includegraphics[width=6.1in]{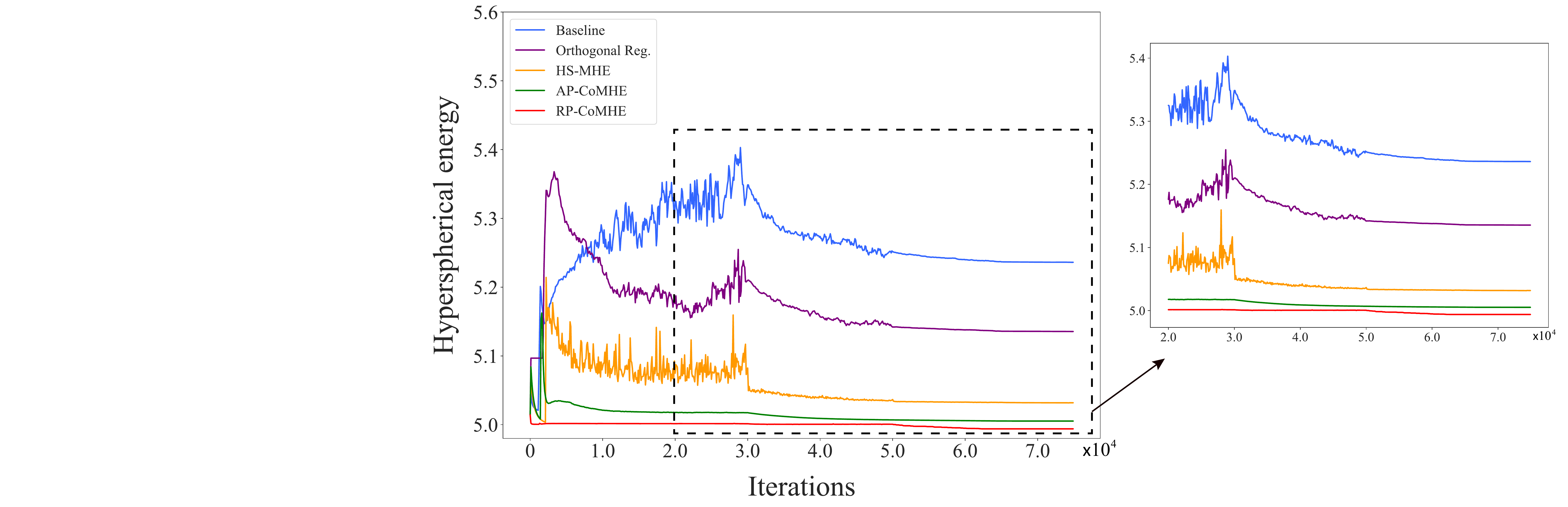}
  \caption{\footnotesize Hyperspherical energy during the entire training. Note that, all networks are initialized with the same weights and therefore have the same hyperspherical energy at the beginning. Note that, ``Orthogonal Reg.'' denotes the orthogonal regularization (use orthogonal constraint to regularize the neurons), which is dramatically different from the rotation/reflection training that is mentioned above and learns orthogonal matrices for neurons.}\label{complete}
\end{figure}

Fig.~\ref{complete} shows the entire training dynamics (from initialization to the end of training) of the hyperspherical energy of baseline CNN and CNN regularized by orthogonal regularization, HS-MHE, AP-CoMHE and RP-CoMHE. All the networks use exactly the same initialized weights to ensure the hyperspherical energy is the same at the beginning. One can observe that the hyperspherical energy is actually very low for the initialized weights. This is because the initialized weights follows Gaussian distribution and the hyperspherical energy is computed with normalized weights. The normalized weights (sampled from Gaussian distribution) follows the uniform distribution on the hypersphere (see Theorem 1 in \cite{OPT2020}), which can obtain the lowest hyperspherical distribution in expectation. However, when the weights of the neural network start to fit the data and minimize the data approximation loss, the neuron weights no longer follow the hyperspherical uniform distribution. Therefore the hyperspherical energy will quickly get large. This is when MHE and CoMHE are useful. From Fig.~\ref{complete}, one can see that without any regularization on hyperspherical energy, the hyperspherical energy of the baseline network gets extremely large at the beginning and then slowly decreases as the training continues. However, the final hyperspherical energy of the baseline network is still way higher than the CNNs regularized by MHE and CoMHE. Notice that, the orthogonality-regularized CNN also obtain high hyperspherical energy at the end (similar to the baseline network). In contrast to MHE, we can observe that CoMHE can effectively minimize the hyperspherical energy and RP-CoMHE achieves significantly lower hyperspherical energy in the end, which well verifies the superiority of the proposed CoMHE.

\newpage
\subsection{Hyperspherical energy dynamics in individual layers}
To demonstrate the hyperspherical energy dynamics in individual layers, we show the hyperspherical energy v.s. iteration in every layer of CNN-9 (as specified in Table~\ref{netarch}) in Fig.~\ref{individual_he}. Since the last fully-connected layer (\ie, classifier layer) is learned from scratch (no rotation training is applied), we do not plot its hyperspherical energy. From the results in Fig.~\ref{individual_he}, we can observe that CoMHE can more effectively minimize the hyperspherical energy in every layer, and RP-CoMHE performs the best in terms of the hyperspherical energy minimization.

\vspace{2mm}
\begin{figure}[h]
  \centering
  \renewcommand{\captionlabelfont}{\footnotesize}
  \setlength{\abovecaptionskip}{5pt}
  \setlength{\belowcaptionskip}{3pt}
  \includegraphics[width=6.76in]{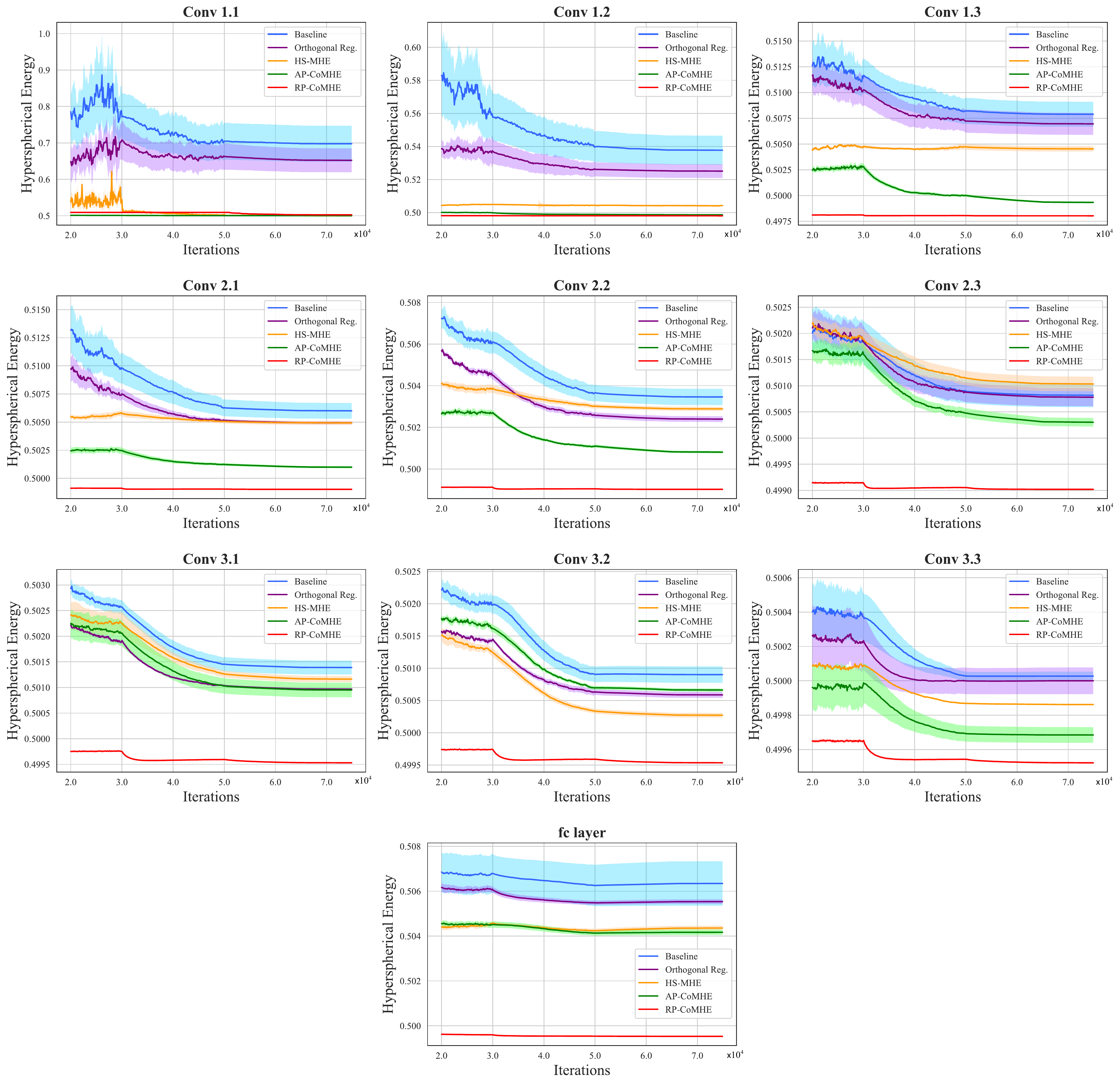}
  \caption{\footnotesize Hyperspherical energy of every layer (Conv1.1, Conv1.2, Conv1.3, Conv2.1, Conv2.2, Conv2.3, Conv3.1, Conv3.2, Conv3.3, fc1) after the 20000-th iteration. Note that, all networks are initialized with the same weights and therefore have the same hyperspherical energy at the beginning.}\label{individual_he}
\end{figure}

\clearpage

\newpage
\section{Additional Exploratory Experiments}\label{extra_explor}

\textbf{Frequency of re-initialization in RP-CoMHE.} In RP-CoMHE, we need to re-initialize the random projections every certain number of iterations to avoid trivial solutions caused by bad initialization. Here, we test how the frequency of re-initialization will affect the accuracy on CIFAR-100, with the projection dimension being 30 and the number of projection being 20. The iteration number being $\infty$ in Table~\ref{freq_rp} represents that the random projection is fixed throughout the training once it is initialized. The results shows the performance is not very sensitive to the frequency of re-initialization, but we cannot fix the random projection during training as it may cause trivial solutions and hurt the performance.

\begin{table}[h]
	\centering
	\footnotesize
	\newcommand{\tabincell}[2]{\begin{tabular}{@{}#1@{}}#2\end{tabular}}
	\renewcommand{\captionlabelfont}{\footnotesize}
	\setlength{\abovecaptionskip}{5pt}
	\setlength{\belowcaptionskip}{0pt}
\begin{tabular}{c || c c c c} 
  \hline
  \# Iterations  & 1 & 200 & 1000 & $\infty$\\
 \hline\hline
 RP-CoMHE & \textbf{24.6} & 24.84 & 24.62  &26.09\\
 \hline
\end{tabular}
\caption{\footnotesize Error of different \# iterations for re-initialization.}
\label{freq_rp}
\end{table}

\textbf{Naively learning projection basis from training data}. We study the case where we enable the back-propagation gradient to flow back to the projection basis. That is to say, the model learns the projection basis naively using training data. We find that naively learning the projection basis yields much worse performance (26.5\%), compared to RP-CoMHE (24.6\%). It is even worse than our baseline half-space MHE (25.96\%). The results show that naively learning projection basis from training data leads to inferior performance. Allowing the projection basis to be updated according to the training data could undermine the strength of CoMHE regularization imposed on the neurons.

\textbf{Shared projection basis.} We take RP-CoMHE as an example to empirically verify the advantages of shared projection basis across different layers. We set the projection dimension to 20 and the number of projections to 30. The plain CNN-9 is used as baseline. Specifically for shared projection basis, we share the random projection basis in Conv1.x, Conv2.x and Conv3.x separately. The shared projection yields 24.6\% error rate. For independent projection basis, we use separated projection basis for different layers and only obtain 26.05\% error rate. The results show that using shared random projection basis for neurons of the same dimensionality improves the network generalization while saving parameters. Note that, all the other experiments use shared projection basis by default.

\newpage
\section{Training Runtime Comparison}\label{exp_time}

We also provide runtime comparison for all the proposed CoMHE. We use the plain CNN-9 for all the methods in this experiment. For RP, we set the projection dimension to 30 and the number of projection to 5. For AP, the number of projection is 1 and the projection dimension is set to 30. This hyperparameter setting for CoMHE can achieve the best testing accuracy on CIFAR-100. The results in Table~\ref{runtime} are computed using the total runtime of runing 100 iterations. We can see that the runtime of RP-CoMHE, AP-CoMHE and Adv-CoMHE is comparable to HS-MHE and the baseline. Without any code optimization, RP-CoMHE is $36\%$ slower than the baseline and $18\%$ slower than the HS-MHE, and AP-CoMHE is $34\%$ slower than the baseline and $17\%$ slower than the HS-MHE. Note that, although CoMHE is relatively slower in terms of training runtime, CoMHE will not  affect the testing runtime of a trained model. That is to say, CoMHE-regularized CNN has the same inference speed with its baseline CNN counterpart. In fact, as long as the training time of CoMHE-regularized CNNs is not geometrically larger than the standard CNN, such computational cost is neglectable in practice and practitioners usually care more about the inference time rather than the training time (CoMHE will not affect the inference time).

\begin{table}[h]
	\centering
	\footnotesize
	\newcommand{\tabincell}[2]{\begin{tabular}{@{}#1@{}}#2\end{tabular}}
	\renewcommand{\captionlabelfont}{\footnotesize}
	\setlength{\abovecaptionskip}{5pt}
	\setlength{\belowcaptionskip}{0pt}
\begin{tabular}{c||c}
  \hline
  \multirow{1}{*}{Method} & Runtime (s)  \\
 \hline\hline
Baseline & 5.61\\
HS-MHE & 6.46\\
RP-CoMHE & 7.62\\
AP-CoMHE & 7.48\\
Adv-CoMHE & 6.37\\
Group CoMHE & 11.12\\
 \hline
\end{tabular}
\caption{\footnotesize Training Runtime (s / 100 iterations) comparison on CIFAR-100.}
\label{runtime}
\vspace{-3mm}
\end{table}

\newpage
\section{Experiments on Graph Convolutional Networks}\label{graph}
We also use CoMHE to improve graph convolutonal networks~(GCN)~\cite{kipf2016semi} for node classification in a graph. We use the official code from \cite{kipf2016semi}\footnote{The code is available at \url{https://github.com/tkipf/gcn}.}, so the experimental setting and hyperparameter setup are exactly the same as \cite{kipf2016semi}. The only difference is that we apply an additional MHE or CoMHE to the weight matrix. Specifically, the graph convolution network uses the following forward model:
\begin{equation}
    \bm{Z}=\textnormal{Softmax}\big(\hat{\bm{A}}\cdot\textnormal{ReLU}(\hat{\bm{A}}\cdot\bm{X}\cdot\bm{W}_0)\cdot\bm{W}_1\big)
\end{equation}
where $\hat{\bm{A}}=\tilde{\bm{D}}^{\frac{1}{2}}\tilde{\bm{A}}\tilde{\bm{D}}^{\frac{1}{2}}$. We note that $\bm{A}$ is the adjacency matrix of the graph,  $\tilde{\bm{A}}=\bm{A}+\bm{I}$ ($\bm{I}$ is an identity matrix), and $\tilde{\bm{D}}=\sum_j\tilde{\bm{A}}_{ij}$. $\bm{X}\in\mathbb{R}^{n\times d}$ is the feature matrix of $n$ nodes in the graph (feature dimension is $d$). $\bm{W}_1$ is the weights of the classifiers. $\bm{W}_0$ is the weight matrix of size $d\times h$ where $h$ is the dimension of the hidden space. We view every column of $\bm{W}_0$ as a neuron, and therefore, there will be $h$ neurons in total. We simply apply MHE or CoMHE to regularize these $h$ neurons. The experimental results are given in Table~\ref{GCN}. We can see from the results that the CoMHE-regularized GCN can consistently outperform the MHE-regularized GCN and the GCN baseline. We use exactly the same code as in the official repository, and the only difference is the regularization on $\bm{W}_0$. We emphasize that CoMHE will not change the inference speed of GCN, so this $1\%-2\%$ performance gain is more like a ``free lunch''.

\begin{table}[h]
	\centering
	\footnotesize
	\newcommand{\tabincell}[2]{\begin{tabular}{@{}#1@{}}#2\end{tabular}}
	\renewcommand{\captionlabelfont}{\footnotesize}
	\setlength{\abovecaptionskip}{5pt}
	\setlength{\belowcaptionskip}{0pt}
\begin{tabular}{c||c c c}
  \hline
  \multirow{1}{*}{Method} & Citeseer & Cora & Pubmed  \\
 \hline\hline
GCN Baseline & 70.3 & 81.3 & 79.0\\
HS-MHE~\cite{LiuNIPS18} & 71.5 & 82.0 & 79.0 \\
RP-CoMHE & \textbf{72.1} & \textbf{82.7} & \textbf{79.5}\\
AP-CoMHE & 72.0 & 82.6 & \textbf{79.5}\\
 \hline
\end{tabular}
\caption{\footnotesize Classification accuracy (\%) of GCN with different hyperspherical energy regularization.}
\label{GCN}
\vspace{-3mm}
\end{table}

\end{appendix}

\end{document}